\documentclass{article}
\usepackage{verbatim}

\usepackage[utf8]{inputenc} 
\usepackage[T1]{fontenc}    
\usepackage{hyperref}       
\usepackage{url}            
\usepackage{booktabs}       
\usepackage{amsfonts}       
\usepackage{nicefrac}       
\usepackage{microtype}      
\usepackage{algorithm}
\usepackage{algorithmic}
\usepackage{microtype}
\usepackage{comment}
\usepackage{graphicx}
\usepackage{subfigure}
\usepackage{booktabs} 
\usepackage[usenames,dvipsnames]{xcolor}
\def\argmax{\mathop{\rm argmax}}
\def\argmin{\mathop{\rm argmin}}

\usepackage{graphicx,amsthm,amsmath, amssymb, enumerate}
\PassOptionsToPackage{numbers, compress}{natbib}

\usepackage{hyperref}

\newtheorem{theorem}{Theorem}

\newtheorem{corollary}[theorem]{Corollary}

\newtheorem{proposition}[theorem]{Proposition}
\newtheorem{remark}{Remark}

\title{Predictive Power of Nearest Neighbors Algorithm under Random Perturbation}

\usepackage[preprint]{nips_2019}
\author{Yue Xing\\
Department of Statistics\\
Purdue University\\
West Lafayette, IN 47907, USA \\
\And
Qifan Song\\
Department of Statistics\\
Purdue University\\
West Lafayette, IN 47907, USA \\
\AND
Guang Cheng\\
Department of Statistics\\
Purdue University\\
West Lafayette, IN 47907, USA \\
}
\begin{document}
\maketitle
	\begin{abstract}
		We consider a data corruption scenario in the classical $k$ Nearest Neighbors ($k$-NN) algorithm, that is, the testing data are randomly perturbed. Under such a scenario, the impact of corruption level on the asymptotic regret is carefully characterized. In particular, our theoretical analysis reveals a phase transition phenomenon that, when the corruption level $\omega$ is below a critical order (i.e., small-$\omega$ regime), the asymptotic regret remains the same; when it is beyond that order (i.e., large-$\omega$ regime), the asymptotic regret deteriorates polynomially. Surprisingly, we obtain a negative result that the classical noise-injection approach will not help improve the testing performance in the beginning stage of the large-$\omega$ regime, even in the level of the multiplicative constant of asymptotic regret. As a technical by-product, we prove that under different model assumptions, the pre-processed 1-NN proposed in \cite{xue2017achieving} will at most achieve a sub-optimal rate when the data dimension $d>4$ even if $k$ is chosen optimally in the pre-processing step. 
	\end{abstract}
	\section{Introduction}
While there has been a great deal of success in the development of machine learning algorithms, much of the success is in relatively restricted domains with limited structural variation or few system constraints. Those algorithms would be quite fragile in broader real-world scenarios, especially when the testing data are contaminated. For example, in image classification, when the input data are slightly altered due to a minor optical sensor system malfunction, a deep neural network may yield a totally different classification result \cite{adversarial2014}; more seriously, attackers can feed well-designed malicious adversarial input to the system and induce wrong decision making \cite{DolphinAttack2017, papernot2017practical}. One strand of existing research along this line focuses on methodology development including generation of corrupted testing samples for which machine learning algorithms fail \cite{papernot2016limitations,papernot2016crafting,grosse2017adversarial} and
design of robust training algorithms \cite{goodfellowexplaining,kurakin2016adversarial,sinha2018certifying, madry2017towards}. The other strand of research focuses on theoretical investigation on how the data corruption affects the algorithm performance \cite{wang2017analyzing,yang2019adversarial,fawzi2016robustness,fawzi2018adversarial}.

		
This work aims to study the robustness of $k$ Nearest Neighbors ($k$-NN) algorithm and its variants, from a theoretical perspective. There is a rich literature on the same topic, but under different setups. 
For example, \citet{cannings2018classification,reeve2019fast,reeve2019classification} study the case where labels for training data are contaminated, and investigate the overall excess risk of trained classifier; \citet{wang2017analyzing,yang2019adversarial} study the case where testing data are contaminated, and only concern about the testing robustness when testing data belong to certain subset of support, rather than the whole support. In contrast to these existing works, the presented study aims to address a different question:
how the {\it overall} regret of $k$-NN classifier, which is trained by uncontaminated training data, is affected when the testing features are corrupted by random perturbation?  
	

Our main theoretical result (derived in the framework of \citealp{ samworth2012optimal})  characterizes the asymptotic regret for randomly perturbed testing data (with an explicit form of multiplicative constant) of $k$-NN with respect to the choice of $k$ and the level of testing data corruption. There are several interesting implications. First, there exists a critical contamination level, (a) below which the asymptotic order of regret is not affected; (b) above which the asymptotic order of regret deteriorates polynomially. Second, although the regret of $k$-NN deteriorates polynomially with respect to the corruption level, it actually achieves the best possible accuracy for testing randomly perturbed data (under a fine tuned choice of $k$). Hence $k$-NN classifier is rate-minimax for both clean data testing task \cite{audibert2007fast,samworth2012optimal,cannings2017local} and randomly perturbed data testing task.

Similar as adversarial training, a strategy to robustify learning algorithm is to inject the same random noises to training data. However, our theoretical analysis reveals that such a noise injection approach doesn't improve the performance of $k$-NN algorithm in the beginning stage of the polynomial deterioration regime, in the sense that it leads to exactly the same asymptotic regret.

Adapting the analysis for adversarial attack scenario (i.e., testing data are contaminated with worst-case perturbation), we compare the regret of $k$-NN under adversarial attack and random perturbation. It is quite interesting to find that when the level of data contamination is beyond the aforementioned critical order, adversarial attack leads to a worse regret than random perturbation {\em only} in terms of the multiplicative constant of the convergence rate. 

		
Our developed theory may also be used to evaluate the asymptotic performance of variants of $k$-NN algorithms. For example, \citet{xue2017achieving} applied $1$NN to pre-processed data (relabelled by $k$-NN) in order to achieve the same accuracy as $k$-NN. Interestingly, this algorithm can be translated into the classical $k$-NN algorithm under a type of perturbed samples to which our theory naturally applies. In particular, we prove that the above algorithm, under our model assumption framework,  only obtain a sub-optimal rate (at least worse than $k$-NN) of regret when $d>4$. 

As far as we are aware, the only related work in the context of $k$-NN is \cite{wang2017analyzing} that evaluated the adversarial robustness of $k$-NN based on the concept of ``robust radius'', 
which is the maximum allowed attack intensity that doesn't affect the testing performance. 
However, their analysis ignores the area nearby the decision boundary where mis-classification is most likely to occur. By filling these gaps, our work attempts to present more comprehensive regret analysis on robustness of $k$-NN. 


		
	\section{Effect of Random Perturbation on $k$-NN}
	In this section, we formally introduce the model setup and present our main theorems which characterize the asymptotic regret for randomly perturbed testing samples. 
    
	\subsection{Model Setup}\label{sec:assumption}
	  
	Denote $P(Y=1|X=x)$ as $\eta(x)$, and its $k$-NN estimator as $\widehat{\eta}_{k,n}(x)$, an average of $k$ nearest neighbors among $n$ training samples. The corresponding Bayes classifier and $k$-NN classifier is defined as $g(x)=1\{\eta(x)>1/2 \}$ and $\widehat{g}_{n,k}(x)=1\{ \widehat{\eta}_{k,n}(x)>1/2 \}$, respectively. 
	
	Define $\omega$ as the level of perturbation. For any intended testing data $x$, we only observe its randomly perturbed version:
	\begin{equation}\label{def:noise}
	\widetilde{x} \sim \mbox{Unif}(\partial B(x,\omega)),
	\end{equation}
	that is, $\widetilde{x}$ is uniformly distributed over $\partial B(x,w)$, the boundary of an Euclidean ball centered at $x$ with radius $\omega$. 

In this case, we define the so called ``perturbed'' regret as
	\begin{eqnarray*}
	\mbox{Regret}(k,n,\omega)=P(Y\neq \widehat{g}_{n,k}(\widetilde X))-P(Y\neq g(X)),
	\end{eqnarray*}
	and $$\mbox{Regret}(n,\omega)=\min\limits_{k=1,...,n}\mbox{Regret}(k,n,\omega).$$
	Note that the $k$-NN classifier $\widehat g_{n,k}$ is trained by uncontaminated training samples. Obviously, when $\omega=0$, the above definition reduces to the common regret that used in statistical classification literature.

The following assumptions are imposed on $X$ and the underlying $\eta$, to facilitate our theoretical analysis.
		\begin{enumerate}
			\item [A.1] $X$ is a random variable on a compact $d$-dimensional manifold $\mathcal X$ with boundary $\partial \mathcal{X}$.  Density function of $X$ is differentiable, finite and bounded away from 0.
			\item[A.2] The set $\mathcal{S}=\{ x|\eta(x)=1/2 \}$ is non-empty. There exists an open subset $U_0$ in $\mathbb{R}^d$ which contains $\mathcal{S}$ such that, define $U$ as an open set containing $\mathcal{X}$, then $\eta$ is continuous on $U\backslash U_0$.
			\item [A.3] There exists some constant $c_x>0$ such that when $|\eta(x)-1/2|\leq c_x$, $\eta$ has bounded fourth-order derivative; when $\eta(x)=1/2$, $\dot{\eta}(x)\neq 0$, where $\dot{\eta}$ is the gradient of $\eta$ in $x$. Also the derivative of $\eta(x)$ within restriction on the boundary of support is non-zero. 
		\end{enumerate}
	 Assumptions A.1 ensures that for any $x\in\mathcal X$, all its $k$ nearest neighbors are close to $x$ with high probability. This is due to the fact that if the density at a point $x$ is positive and finite, its distance to its $k$th nearest will be of $O_p((k/n)^{1/d})=o_p(1)$. 
	 Assumption A.2 ensures the existence of $x$ in $\{x\in\mathcal{X}|\; \eta(x)=1/2\}$ and $\eta(x)$ is continuous in other regions of $\mathcal{X}$. Assumption A.3 on $\eta(x)$ is slightly stronger than that imposed in \cite{samworth2012optimal} due to the consideration of testing data contamination. Specifically, the additional smoothness on $\eta(x)$ imposed in Assumption A.3 guarantees that some higher-order terms in the Taylor expansion of $\mathbb{E}\{\widehat{\eta}_{k,n}(\widetilde{x})-\eta(x)\}$ are negligible.
	
	\subsection{Asymptotic Regret and Phase Transition Phenomenon }\label{sec:comparison}
	
We are now ready to conduct regret analysis for $k$-NN in the presence of randomly perturbed testing examples. For any $x\in\mathcal X$, define  
\[
 t(x)= \mathbb{E}\big(\|X_i-x\|_2^2\, \big|\,X_i \mbox{ is among the $k$ nearest neighbors of } x\big).
\]
Therefore, $t(x)$ represents the expected squared distance from $x$ to any of its $k$ nearest neighbors.
And take $t=\max_x t(x)$. Also denote $\bar{f}(x,y)$ and $\bar f(x)$ as the joint density of $(x,y)$ and marginal density of $x$ respectively. Let $f_1(x):=\bar{f}(x,0)$, $f_2(x):=\bar{f}(x,1)$, and $\Psi(x):=f_1(x)-f_2(x)$.

We first characterize the asymptotic perturbed regret.

\begin{theorem}\label{thm:noise}
Define $\epsilon_{k,n,\omega}=\max(\log k/\sqrt{k},t+\omega)$. Under [A.1] to [A.3] if testing data is randomly perturbed, 
		then it follows that
		\begin{equation}\label{eqn:noise}
		\begin{split}
		\mbox{Regret}(k,n,\omega)=&\underbrace{\frac{1}{2}\int_{\mathcal{S}}\frac{\|\dot{\Psi}(x_0)\|}{\|\dot{\eta}(x_0)\|^2}\left(b(x_0)t(x_0)\right)^2d\text{Vol}^{d-1}(x_0)}_{Bias} +\underbrace{\frac{1}{2}\int_{\mathcal{S}}\frac{\omega^2}{d}\|\dot{\Psi}(x_0)\| d\text{Vol}^{d-1}(x_0)}_{Corruption}\\&+\underbrace{\frac{1}{2}\int_{\mathcal{S}}\frac{1}{4k}\frac{\|\dot{\Psi}(x_0)\|}{\|\dot{\eta}(x_0)\|^2}d\text{Vol}^{d-1}(x_0)}_{Variance}+\text{Remainder},
		\end{split}
		\end{equation}
		where Remainder=$O(\epsilon_{k,n,\omega}^3)$ as $k, n\to\infty$. The term $b(\cdot)$ relies on the true $\eta(x)$ and the distribution of $X$, and does not change with respect to $k$ and $n$:
		 \begin{eqnarray*}
		 	b(x)&=&\frac{1}{\bar{f}(x)d}\left\{ \sum_{j=1}^d [\dot{\eta}_j(x)\dot{\bar{f}}_j(x)+\ddot{\eta}_{j,j}(x)\bar{f}(x)/2 ] \right\}.
		 \end{eqnarray*} 
Here $\dot{\eta}$, $\ddot{\eta}$, and $\dot{\bar{f}}$ represent the gradient, Hessian of $\eta$, and gradient of $\bar{f}$ respectively. The subscript $j$ denotes the $j$'th element of $\dot{\eta}$ or $\dot{\bar{f}}$, as well as subscript $j,j$ denotes the $(j,j)$'th element of $\ddot{\eta}$.
	\end{theorem}
Our result (\ref{eqn:noise}) clearly decomposes the asymptotic regret into squared bias term, data corruption effect term, variance term as well as a reminder term due to higher-order Taylor expansion.
The first three terms are of order $O((k/n)^{4/d})$, $O(\omega^2)$ and $O(1/k)$ respectively, and the reminder term is technically derived from high order Taylor expansion and Berry-Essen theorem. When $k$ is within a reasonable range, the reminder term is negligible comparing with the other three terms. 
When $\omega=0$, (\ref{eqn:noise}) reduces to the bias-variance decomposition commonly observed in the nonparametric regression literature. 

Based on Theorem \ref{thm:noise}, through changing $\omega$, we have the following observations:

\paragraph{Phase Transition Phenomenon} We discover a phase transition phenomenon for the regret with respect to the level of testing data contamination in general.


\begin{enumerate}
    \item When $\omega^2\preccurlyeq (1/k\wedge t^2)$
\footnote{To prevent the conflict of definitions of $\omega$, we use $\preccurlyeq$ and $\succcurlyeq$ to replace $o(.)$ and $\omega(.)$ in O/$\Omega$ notation.  Moreover, for $a(n)\preccurlyeq b(n)\preccurlyeq 1$, we mean that $b(n)/1<n^{-\varepsilon_1}$ and $a(n)/b(n)<n^{-\varepsilon_2}$ for some $\varepsilon_1,\varepsilon_2>0$  when $n\rightarrow\infty$.}, the data corruption has almost no effect: $\mbox{Regret}(k,n,\omega)/\mbox{Regret}(k,n,0)\rightarrow 1$;

\item When $\omega^2=\Theta(1/k\vee t^2)$, the regret is of the same order as $\mbox{Regret}(k,n,0)$ but with a different multiplicative constant depending on $\bar{f}$ and $\eta$;

\item When $\omega^2\succcurlyeq (1/k\vee t^2)$, 
	$\mbox{Regret}(k,n,\omega)=\Theta(\omega^2)$ and $ \mbox{Regret}(k,n,\omega)\succcurlyeq\mbox{Regret}(k,n,0)$.

\end{enumerate}	

	

	\paragraph{Impact on $\mbox{Regret}(n,\omega)$ and the choice of $k$}
	The value $k$ plays an important role in the $k$-NN algorithm. It is essential to examine how the intensity level $\omega$ affects the corresponding  $\mbox{Regret}(n,\omega)$. From Theorem \ref{thm:noise}, one can derive that for randomly perturbed testing scenario, if $\omega\preccurlyeq n^{-2/(d+4)}$, $\mbox{Regret}(n,\omega)=\Theta(n^{-4/(d+4)})$; if $\omega\succcurlyeq n^{-2/(d+4)}$, $\mbox{Regret}(n,\omega)=\Theta(\omega^2)$. In other words,  $\mbox{Regret}(n,\omega)=\Theta(\omega^2\vee n^{-4/(d+4)})$.
	The above rate can be achieved if we choose $k=\Theta(n^{4/(4+d)})$ when $\omega\preccurlyeq n^{-4/3(4+d)}$ and $ 1/\omega^2 \preccurlyeq k \preccurlyeq n\omega^{d/2} $ when $\omega \succcurlyeq n^{-4/3(4+d)}$.

	\paragraph{Robustness Trade-off in $k$-NN}
	Theorem \ref{thm:noise} reveals that the regret is of order $O((k/n)^{4/d}+1/k)+O(\omega^2)$, therefore the data corruption has no impact as long as $\omega^2=o((k/n)^{4/d}+1/k)$.
	In other words, if $k$ is chosen to be optimal and minimizes the regret for uncontaminated testing sample, i.e., $(k/n)^{4/d}+1/k$ is small, then the $k$-NN is more sensitive to the increase of $\omega^2$. On the other hands, if $k$ is some sub-optimal choice such that $(k/n)^{4/d}+1/k$ is larger, then $k$-NN is more robust to the testing data corruption. Hence there is a trade-off between the accuracy of uncontaminated testing task and robustness with respect to random perturbation corruption.
	

    \paragraph{Effect of Metric of Noise} Note that $\widetilde{x}$ can be defined on $\mathcal{L}_p$ ball / sphere for different $p\geq 1$. Theorem \ref{thm:noise} reveals that the effect of $\omega$ is independent with $t$ and $1/k$ when $\omega_{k,n,\omega}^3$ is not dominant while $\widetilde{x}$ is uniformly distributed in a ball / sphere (so that first order terms w.r.t. noise are zero). As a result, define $\varepsilon_{p}$ as a random variable uniformly distributed in a $\mathcal{L}_p$ ball/ sphere, then to change the regret accordingly, one can replace $\omega^2/d$ in Theorem \ref{thm:noise} into $\frac{1}{\|\dot{\eta}(x_0)\|^2}\mathbb{E}(\varepsilon_p^{\top}\dot{\eta}(x_0))^2 $.
	
	\paragraph{Minimax Result under General Smoothness Conditions} To relax the strong assumptions A.1 to A.3, we follow \cite{CD14} to impose smoothness and margin conditions. Basically, the observations are similar as the results above. One can also show that $k$-NN reaches the optimal rate of convergence.
	\begin{theorem}[Informal Statement under General Smoothness Conditions]\label{thm:informal}
	If the distribution of $(X,Y)$ satisfies
    \begin{enumerate}
        \item $|\eta(x)-\eta(x')|\leq A\|x-x'\|^{\alpha}$ for all $x$;
        \item $P(|\eta(X)-1/2|<t)\leq Bt^{\beta}$ for some $\beta>0$;
    \end{enumerate}
	    together with some other general assumptions, then when taking $$k\asymp O(n^{2\alpha/(2\alpha+d)}\wedge (n^{2\alpha/d}\omega^{-2\alpha\beta})^{1/(2\alpha/d+\beta+1)}),$$
	    the regret becomes
		\begin{eqnarray*}
		\mbox{Regret}(n,\omega)=O\left( \omega^{\alpha(\beta+1)}\vee n^{-\alpha(\beta+1)/(2\alpha+d)}\right),
		\end{eqnarray*}
		which is proven to be the minimax rate.
	\end{theorem}
	Formal assumptions and results for Theorem \ref{thm:informal} are postponed in Section C in Appendix (Theorem S.3 for convergence rate and Theorem S.4 for minimax rate). From Theorem \ref{thm:informal}, the rate of regret is dominated by the larger one between the random perturbation effect ($\omega^{\alpha(\beta+1)}$) and the minimax rate for clean data ($n^{-\alpha(\beta+1)/(2\alpha+d)}$). 
	Similar as in \cite{samworth2012optimal},
	our regret result derived under conditions A.1-A.3 matches the minimax rate of Theorem \ref{thm:informal} by taking $\alpha=2$ and $\beta=1$.

\begin{remark}[Adversarial Data Corruption]
So far, we focus on the case of random perturbation. As a by-product, we analyze the effect of some special non-random data corruption. Due to page limit constraint, the detailed results and discussions are postponed to Section A in appendix. In general, comparing with worst-case data corruption, the effect of random perturbed noise will mostly cancel out in each perturbation direction, thus leading to a smaller corruption effect term. Therefore, unsurprisingly, $k$-NN is more robust to random perturbed data corruption than worse-case adversarial data corruption. However, our rigorous analysis shows that the regret under adversarial data corruption (defined formally in Appendix) is still of the same order as in the case of random perturbation but with a larger multiplicative constant when $\omega\succcurlyeq n^{-2/(d+4)}$ (when $\omega\preccurlyeq n^{-2/(d+4)}$, the effect of either adversarial corruption / random perturbation is negligible).    
\end{remark}

\subsection{Comparison with Noise-Injected $k$-NN}\label{sec:robust}
   

In an iterative adversarial training algorithm (e.g. \citealp{sinha2018certifying}),  attack is designed in each iteration based on the current model, and the model is updated based on attacked training data. Similarly for random perturbation, a natural idea to enhance the robustness of $k$-NN is to inject noise in training data so that training and testing data share the same distribution. One can  randomly perturb training samples and train $k$-NN using the perturbed training data set. Comparing this noise-injection $k$-NN and traditional $k$-NN methods, we find that there is no performance improvement for the former even when the corruption level is in the early stage of the polynomial deterioration regime. 

Denote $\widetilde{g}(\widetilde{x}):=P(Y=1|\widetilde{x}\text{ is observed})$ as the Bayes estimator and $\widehat{g}'_{n}$ as the estimator trained using randomly perturbed training data. Let both estimators $\widehat{g}_n$ and $\widehat{g}'_n$ adopt their best choices of $k$ respectively. Then we have
\begin{theorem}\label{thm:no_diff}
    Under [A.1] to [A.3], when $0<\omega^3\preccurlyeq n^{-4/(d+4)}$,
    \begin{eqnarray}\label{eqn:compare}
    \frac{P(Y\neq\widehat{g}_{n}(\widetilde{x}))-P(Y\neq \widetilde{g}(\widetilde{x}))}{ P(Y\neq\widehat{g}_{n}'(\widetilde{x}))-P(Y\neq \widetilde{g}(\widetilde{x})) }\rightarrow 1.
    \end{eqnarray}
\end{theorem}

Although it is intuitive to consider perturbing training data such that they match the distribution of the corrupted testing data,
result (\ref{eqn:compare}) implies that the estimators $\widehat g_n$ and $\widehat g_n'$ asymptotically share the same predictive performance for randomly perturbed testing data, and noise injection strategy is futile.
Note that this result holds when $\omega$ is small. Combined with our analysis in Theorem \ref{thm:noise}, within the range $n^{-2/(d+4)}\preccurlyeq\omega\preccurlyeq
n^{4/3(d+4)}$, the regret is deteriorated polynomially due the impact of testing data corruption, and the deterioration can not be remedied by noise-injection adversarial training at all.  One heuristic explanation is that, 
 such an injected perturbation may introduce additional noise to the estimation procedure and change some underlying properties (e.g. smoothness), and consequently this strategy of perturbing training data doesn't necessarily help to achieve smaller regret especially when $\omega$ is small.


\begin{remark}
    The above robustness result can be applied after we obtain knowledge while testing. In reality, we do not know whether the testing data is corrupted or not. In such a case, other techniques, e.g. hypothesis testing, are needed to be applied while testing.
\end{remark}

	\section{Application to other Nearest Neighbor-Type Algorithms}

Our theoretical analysis may be adapted to other NN-type algorithms: pre-processed 1NN  \citep{xue2017achieving} and distributed-NN \citep{duan2018DNN}. In particular, we prove that the regret of the former is sub-optimal for some class of distributions, and explain why the regret of the latter converges in the optimal rate, both from the random perturbation viewpoint.

		
	
	\subsection{Pre-processed 1NN}

	In some literature \cite{xue2017achieving,wang2017analyzing}, the algorithms run 1NN to do prediction using pre-processed data instead of running $k$-NN using raw data. The pre-processing step (or called de-noising step) is reviewed in Algorithm \ref{alg:pre}. Specifically, we firstly run $k$-NN to predict labels for the training data set, then replace the original labels with the predict labels $\widehat{y}_i$'s. In this way, applying 1NN on data $(x_1,\widehat{y}_1)$,...,$(x_n,\widehat{y}_n)$ can achieve good accuracy  while the computational cost is much smaller than $k$-NN. 
	
		\begin{algorithm}[tb]
			\caption{Data Pre-processing}
			\label{alg:pre}
			\begin{algorithmic}
				\STATE {\bfseries Input:} data $(x_1,y_1)$,..., $(x_n,y_n)$, number of neighbors $k$.
				\FOR{$i=1$ {\bfseries to} $n$}
				\STATE Find the $k$ nearest neighbors of $x_i$ in $x_1$,...,$x_n$, excluding $x_i$ itself. Denote the index set of these $k$ neighbors as $N_i$.
				\STATE Estimate a label for $x_i$ as
				\begin{eqnarray*}
				\widehat{\eta}(x_i)&=&\frac{1}{k}\sum_{j\in N_i}y_j,\\
				\widehat{y}_i&=& 1_{\{ 	\widehat{\eta}(x_i)>1/2 \}}.
				\end{eqnarray*}
				\ENDFOR
				\STATE{\bfseries Output:} $(x_1,\widehat{y}_1)$,...,$(x_n,\widehat{y}_n)$.
			\end{algorithmic}
 		\end{algorithm}
	This in fact can be treated as an application of random perturbation of testing data in $k$-NN, in the sense that this classifier can be equivalently represented as $k$-NN under perturbed sample:
	\[\widehat{g}_{1\rm NN}(x) = \widehat{g}_{n,k}( \widetilde{x}),\]
	where $\widehat{g}_{1\rm NN}$ is the pre-processed 1NN classifier, and $\widetilde{x}$ is the perturbed observation of $x$, which is the nearest neighbor of $x$.
	Although the $\widetilde x$ here doesn't exactly match our definition (\ref{def:noise}), it still can be viewed as randomly perturbed $x$ with level of contamination $\omega =\Theta(n^{-1/d}) $, which is the order for the expected length from $x$ to the nearest neighbor under Assumption A.1.
	
From this point of view, Theorem \ref{thm:noise} can be be applied to derive the regret of the pre-processed 1NN algorithm, whose rate of convergence turns out to be slower than the optimal rate $\Theta(n^{-4/(d+4)})$ of $k$-NN when the data dimension $d$ is relatively high, say $d>4$. 
	
	
	\begin{theorem}\label{thm:1nn}
        Under [A.1] to [A.3], the regret of pre-processed 1NN under un-corrupted testing data is
        \begin{equation*}
        \begin{split}
          \textit{Regret}_{\textit{1NN}}(k,n) =&\frac{1}{2}\int_{\mathcal{S}}\frac{\|\dot{\Psi}(x_0)\|}{\|\dot{\eta}(x_0)\|^2}\left(b(x_0)t(x_0)\right)^2 d\text{Vol}^{d-1}(x_0)+\frac{1}{2}\int_{\mathcal{S}}\frac{1}{4k}\frac{\|\dot{\Psi}(x_0)\|}{\|\dot{\eta}(x_0)\|^2}d\text{Vol}^{d-1}(x_0)\\&+ \text{Corruption} + \text{Remainder},
          \end{split}
        \end{equation*} 
        where \begin{eqnarray*}
        \text{Corruption} &=& \Theta(n^{-2/d}),\\
        \text{Remainder}&=&o(n^{-2/d})
        \end{eqnarray*}
        when both $1/k$ and $(k/n)^{4/d}$ are of $O(n^{-1/d})$, and $k=O(n^{6/d})$. As a result, pre-processed 1NN is sub-optimal when $d>4$ (compared with optimal rate $n^{-4/(d+4)}$).
\end{theorem}

	The result in Theorem \ref{thm:1nn} reveals a sub-optimal rate for the pre-processed $1$NN, in comparison with the optimal rate claimed by \cite{xue2017achieving}. This is not a contradiction, but due to different assumptions imposed for $\eta$ and the distribution of $X$. \citet{xue2017achieving} assumes $\alpha$-smoothness condition $|\eta(x)-\eta(x')|\leq A\|x-x'\|^{\alpha}$ which is more general than our smoothness assumption A.3.

	\subsection{Distributed-NN}
The computational complexity of $k$-NN is huge if $n$ is large, therefore we consider a distributed NN algorithm: we randomly partition the original data into $s$ equal-size parts, then given $x$, for each machine, the $k/s$ nearest neighbors of $x$ are selected and calculate $\widehat{\eta}_j(x)$ for $j=1,...,s$, finally we average $\widehat{\eta}_1(x)$,...,$\widehat{\eta}_s(x)$ to obtain $\widehat{\eta}(x)$. The algorithm is shown in Algorithm \ref{alg:dist} as in \cite{duan2018DNN}.
		
			\begin{algorithm}[tb]
				\caption{Distributed-NN}
				\label{alg:dist}
				\begin{algorithmic}
					\STATE {\bfseries Input:} data $(x_1,y_1)$,..., $(x_n,y_n)$, number of neighbors $k$, number of slaves $s$, a point $x$ for prediction.
					\STATE Randomly divide the whole data set into $s$ parts, with index sets $S_1$,...,$S_s$.
					\FOR{$i=1$ {\bfseries to} $s$}
					\STATE Find the $k/s$ nearest neighbors of $x$ in $\{ x_j\;|\; j\in S_i \}$. Denote the index set of these $k/s$ neighbors as $N_i$.
					\STATE Estimate
					\begin{eqnarray*}
						\widehat{\eta}_i(x)&=&\frac{1}{k/s}\sum_{j\in N_i}y_j.
					\end{eqnarray*}
					\ENDFOR
					\STATE Estimate the label of $x$ as
					\begin{eqnarray*}
						\widehat{\eta}(x)&=&\frac{1}{s}\sum_{i=1}^s 	\widehat{\eta}_i(x),\\
						\widehat{y}&=& 1_{\{ 	\widehat{\eta}(x)>1/2 \}}.
					\end{eqnarray*}
					\STATE{\bfseries Output:} $(x,\widehat{y})$.
				\end{algorithmic}
			\end{algorithm}

Distributed-NN is practically different from $k$-NN in a single machine since the $k$ selected neighbors aggregated from $s$ subsets of data are not necessarily the same $k$ nearest neighbors selected in a single machine. Therefore, an additional assumption $k/s\rightarrow\infty$ is imposed, to ensure that the neighborhood set selected by distributed NN behaves similarly to the neighborhood set selected by single machine $k$-NN, in the sense that $E\|X_i-x\|^2$, where $X_i$ belongs to the distributed NN neighborhood set, is of the same order of $t(x)$. 
Therefore, based on Theorem 1,  we obtain that the order of regret of Distributed-NN is in fact of the same order as $k$-NN. Formally, we have the following corollary:
	\begin{corollary}
	Under [A.1] to [A.3], when $\int_{\mathcal{S}}\|\dot{\Psi}(x_0)\|d Vol^{d-1}(x_0)\neq 0$ and $P(b(X)>0|\eta(X)=1/2)>0$, if the number of machines $s\preccurlyeq k$, then 
	\begin{equation*}
	   \mbox{Regret}_{\rm DNN}(k,n) = \Theta(\mbox{Regret}_{\rm kNN}(k,n)).
	\end{equation*}
	where $\mbox{Regret}_{\rm DNN}$ and $\mbox{Regret}_{\rm kNN}$ denote the (clean testing data) regret of distributed NN and $k$-NN algorithms respectively.	\end{corollary}

	\section{Numerical Experiments}
	In Section \ref{sec:sim1}, we will evaluate the empirical performance of $k$-NN algorithm
	for randomly perturbed testing data, where we compare the $k$-NN classifiers 
	trained by raw un-corrupted training data and trained by noise injected training data (i.e., $\widehat g_n$ versus $\widehat g_n'$ defined in Section \ref{sec:robust}). 
	In Section \ref{sec:sim2}
	We also conduct experiments to compare $k$-NN with pre-processed 1NN for un-corrupted testing data. These numerical experiments are intended to show: (i) $k$-NN trained by un-corrupted training data has a similar testing  performance with $k$-NN trained by noise injected training data when $\omega$ is small, which validates our Theorem \ref{thm:no_diff};  (ii) for un-corrupted testing data, the regret of pre-processed 1NN is much worse than that of $k$-NN if $d>4$, which validates our Theorem \ref{thm:1nn}. In general, $k$-NN is expected to be always better than pre-processed 1NN.

Note that in all our real data applications except for MNIST data set, we normalized all attributes.	
	

	\subsection{Tackling Random Perturbation}\label{sec:sim1}
	\subsubsection{Simulation}
    The random variable $X$ is of $5$ dimension, and each dimension independently follows exponential distribution with mean 0.5. The conditional mean of $Y$ is defined as
	\begin{eqnarray}\label{eqn:simualtion_model}
		\eta(x)&=&\frac{e^{x^{\top}w}}{e^{x^{\top}w}+e^{-x^{\top}w}},\\\qquad
		w_i&=& i-d/2\qquad i=1,...,d.\nonumber
	\end{eqnarray}
For each pair of $(k,n)$, we use $2^6,...,2^{11}$ training samples, 10000 testing samples and repeated 50 times to calculate the average regret. In each repetition, 5-fold cross validation was used to obtain $\widetilde{k}$. 
Then based on \cite{samworth2012optimal}, we adjust the number of neighbors to $\widehat{k}=\widetilde{k}(5/4)^{4/(4+d)}$ since $\widetilde{k}$ is the best $k$ value for $4n/5$ samples instead of $n$ samples. 
The two classifiers, trained via 
un-corrupted training data and corrupted training data respectively, used to predict corrupted testing data.

From Figure \ref{fig:noise_simu}, as the number of training samples increases, the regret for both $k$-NNs gets reduced for $0<\omega\leq 0.05$ in the same speed. This verifies that these two $k$-NNs in fact do not differ a lot when $\omega$ is small, i.e., Theorem \ref{thm:no_diff}. Empirically, the regret of $k$-NN trained by corrupted training data is worse than the one trained by un-corrupted training data when $\omega\leq 0.5$. On the other hand, when $\omega$ is large (such that required condition in Theorem \ref{thm:no_diff} fails), the two $k$-NNs may perform significantly different. For example, we tried $\omega=3$, when sample size $n=64$, $\log_2(\mbox{Regret})$ is -2.86 using uncontaminated data, and is -3.11 using corrupted training data. 

\begin{figure}[!ht]
    \centering
    \includegraphics[scale=0.41]{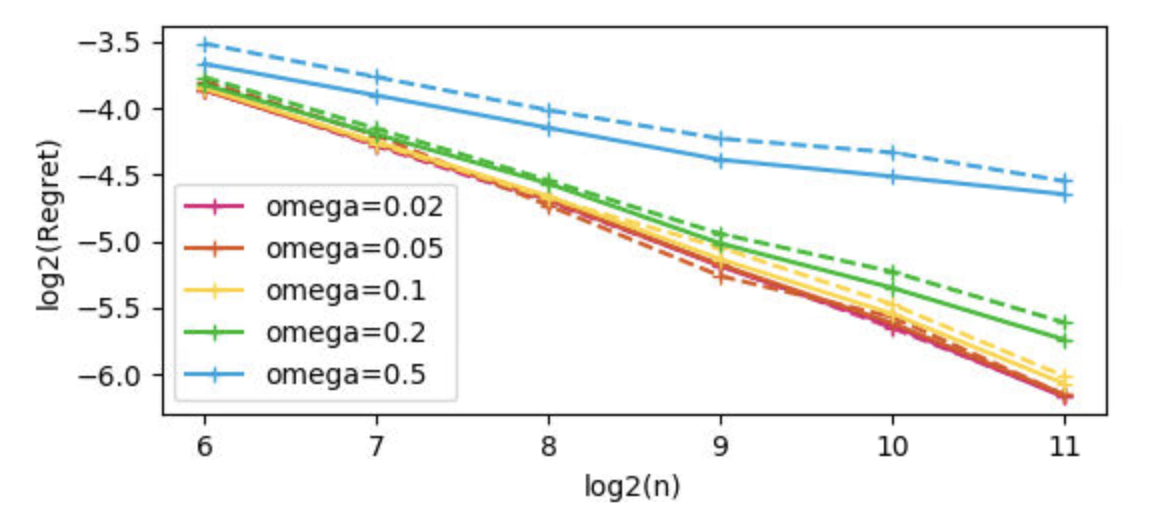}
    \caption{ Comparison between $k$-NN trained by raw training data (solid line) and $k$-NN trained by noise injected training data (dashed line) in Simulation. 
    }
    \label{fig:noise_simu}
\end{figure}

\subsubsection{Real Data}
We use two real data sets for the comparison of 2 $k$-NNs: Abalone (\citealp{Dua:2019}), HTRU2 (\citealp{lyon2016fifty}). For Abalone data set, the data set contains 4177 samples, and two attributes (\textit{Length} and \textit{Diameter}) are used in this experiment. The classification label is whether an abalone is older than 10.5 years. For HTRU2 data set \cite{lyon2016fifty}, the data has a size of of 17,898 with 8 continuous attributes. For each data set, 25\% of the samples are contaminated by random noise, and thereafter are used as testing data.

As shown in Figure \ref{fig:noise_real}, when $\omega$ is small, for both data sets, the error rate (misclassification rate) of the two $k$-NNs do not differ a lot when $\omega$ is small. 

\begin{figure}[ht!]
    \centering
    \includegraphics[scale=0.35]{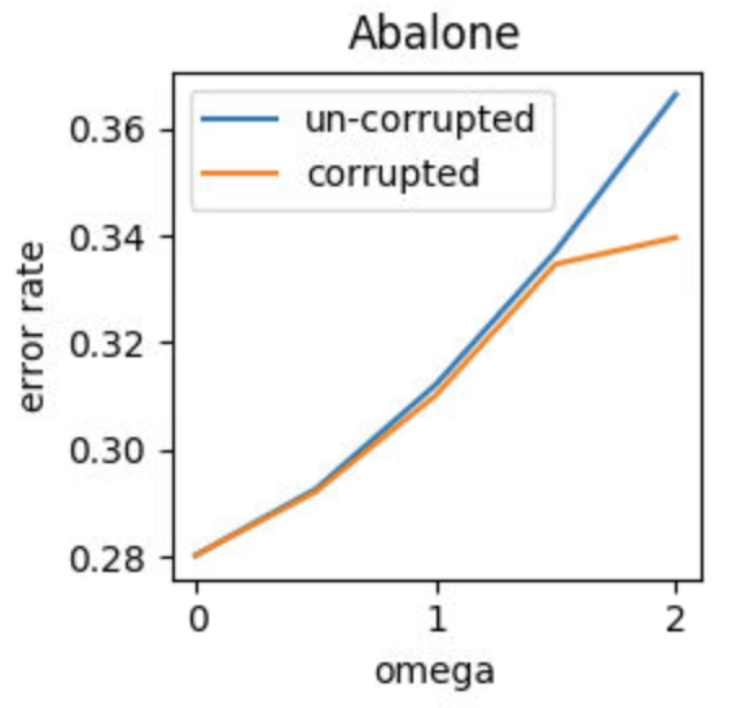}
    \includegraphics[scale=0.35]{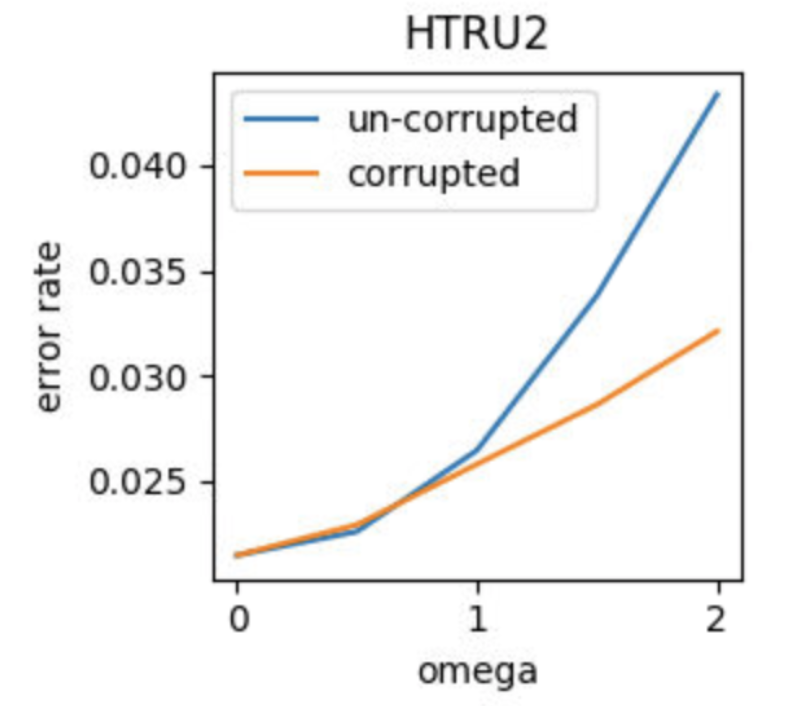}
    \caption{$k$-NN trained by raw Training Data vs noise injected Training Data}
    \label{fig:noise_real}
\end{figure}

	\subsection{Performance of 1NN with Pre-processed Data}\label{sec:sim2}
	\subsubsection{Simulation}

		\begin{figure}[!ht]
	    \centering
	    \includegraphics[scale=0.41]{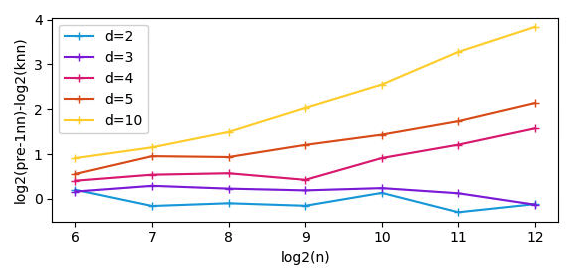}
	    \caption{Simulation Comparison between $k$-NN and pre-processed 1NN, the $y$ axis denotes the $\log_2(\mbox{Regre of pre-processed 1NN})-\log_2(\mbox{Regre of $k$NN})$}
	    \label{fig:knn_better}
	\end{figure}	
	To observe a clear difference, instead of $w_i$ in  (\ref{eqn:simualtion_model}), we use a model where each dimension of $x$ follows uniform $[0,1]$ distribution, with $\eta$ in (\ref{eqn:simualtion_model}), and
	\begin{eqnarray}\label{eqn:simualtion_model_1nn}
		w_i&=& i-d/2-0.5\qquad i=1,...,d.\nonumber
	\end{eqnarray}
	for different values of $d$ to compare the performance between $k$-NN and pre-processed 1NN. $X$ now follows $d$-dimensional uniform $(0,1)$. From Figure \ref{fig:knn_better}, we show that the order of regret of pre-processed 1NN is different from that of $k$-NN when $d\geq 4$. 

	We also replace the uniform distribution as exponential distribution with mean 0.5 for each dimension of $x$. Figure \ref{fig:1nn_exp} shows that for $d=5$ and $d=10$, the increasing trend of regret ratio is obvious, indicating a sub-optimal rate of pre-processed 1NN.

\begin{figure}[!ht]
    \centering
    \includegraphics[scale=0.41]{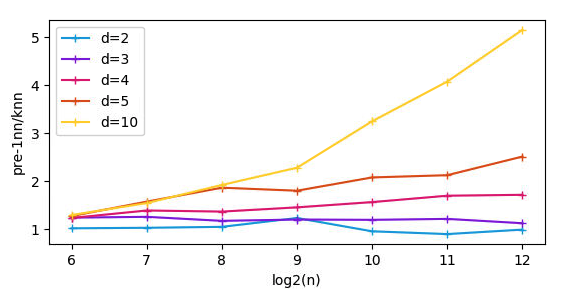}
    \caption{ Comparison between $k$-NN and pre-processed 1NN. Each Dimension of $X$ Independently Follows Exponential Distribution with Mean 0.5. Y Represents the Regret Ratio (instead of logarithm of Regret Ratio).}
    \label{fig:1nn_exp}
\end{figure}
	
	
	\subsubsection{Real Data}
	
	We use four data sets to compare the $k$-NN and pre-processed 1NN: MNIST, Abalone, HTRU2, and Credit (\citealp{yeh2009comparisons}).
	
	For MNIST, this data set contains 70000 samples and data dimension is 784.
	We randomly pick 25\% samples as testing data, and randomly pick $2^{i}$ ($i=7,8,...,12$) samples to train $k$-NN classifier and pre-processed 1NN classifier, where the choices of $k$ for both algorithms are determined by 5-folds cross validation. We repeated this procedure 50 times with different random seeds to obtain the mean testing error rates. 
	As shown in Figure \ref{fig:knn_better_mnist}, through increasing the number of training samples, the error rate ratio between pre-process 1NN classifier and $k$-NN classifier is stably above 1 and is around 1.17.
    
    	\begin{figure}[!ht]
	    \centering
	    \includegraphics[scale=0.46]{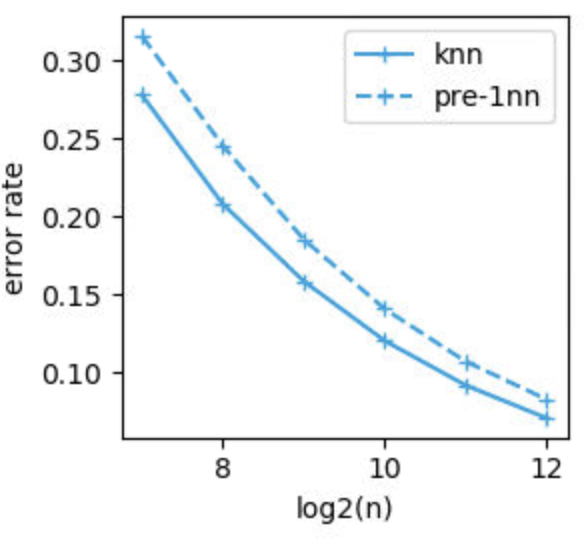}
	    \includegraphics[scale=0.45]{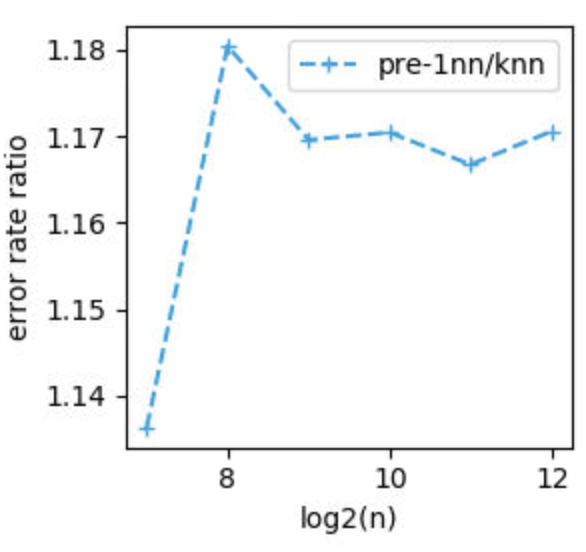}
	    \caption{MNIST, $k$-NN vs pre-processed 1NN}
	    \label{fig:knn_better_mnist}
	\end{figure}
	
For Abalone data set, we conducted experiment in the same way as MNIST, and observe that the error rate using pre-processed 1NN is always greater than $k$-NN. As is shown in Figure \ref{fig:my_label}, while the error rate for both pre-processed 1NN and $k$-NN are decreasing in $n$, their difference changes little when $n\leq 2^{11}$. 
	
	\begin{figure}[!ht]
    \centering
    \includegraphics[scale=0.42]{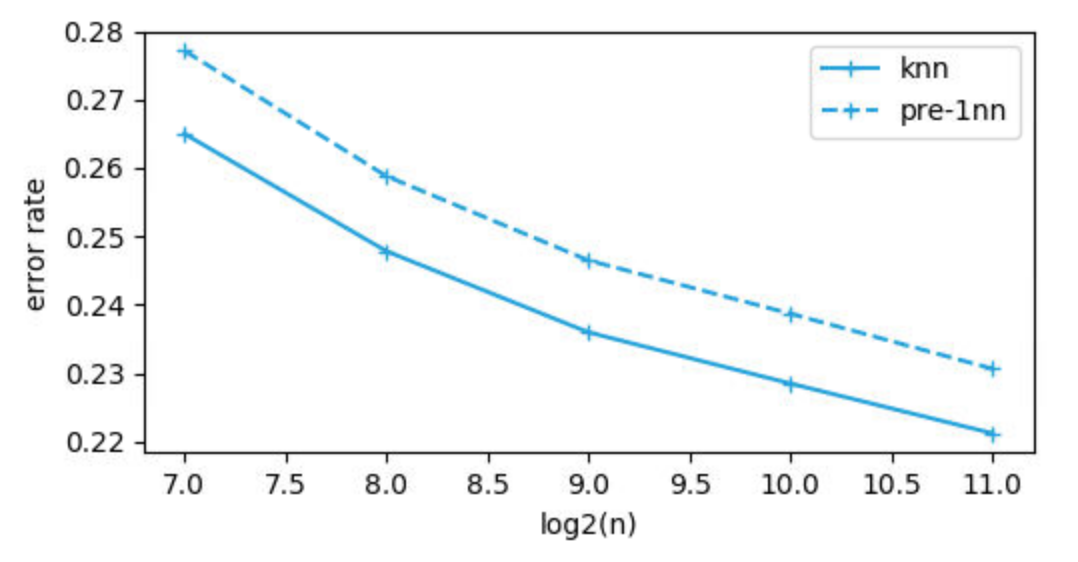}
    \caption{Comparison between kNN and Pre-processed 1NN in Abalone Data Set}
    \label{fig:my_label}
\end{figure}

	In Credit data set, there are 30000 samples (25\% as testing data) with 23 attributes. For HTRU2 and Credit data set, the mean and standard error of error rates in the 50 repetitions are summarized in Table \ref{tab:real_1nn}. From Table \ref{tab:real_1nn}, using $k$-NN we obtain slightly smaller error rate on average.
	    
		\begin{table}[!ht]
		\centering
	    \begin{tabular}{|c|c|c|}
	         \hline Data Set& Error Rate ($k$-NN) &  Error Rate (Pre-1NN)  \\\hline
	         
	        Credit & 0.18879 & 0.1899 \\\hline  
	         HTRU2 &  0.02143 & 0.0221\\\hline
	    \end{tabular}\\
	    \caption{Comparison between $k$-NN and Pro-processed 1NN (Pre-1NN) in HTRU2 and Credit}
	    \label{tab:real_1nn}
	\end{table}

	\section{Conclusion and Discussion}
	In this work, we conduct asymptotic regret analysis of $k$-NN classification for randomly perturbed testing data. In particular, a phase transition phenomenon is observed: when the corruption level is below a threshold order, it doesn't affect the asymptotic regret; when the corruption level is beyond this threshold order, the asymptotic regret grows polynomially. Moreover, when the level of corruption is small, there is no benefit to perform noise injected adversarial training approach.
	
	Moreover, using the idea of random perturbation, we can further explain why pre-processed 1NN converges in a sub-optimal rate: it can be treated as $k$-NN with perturbation in testing data while $\omega$, the distance from $x$ to its nearest neighbor, is large when $d>4$. It is worth to mention that our analysis can be applied to Distributed-NN to verify the optimal rate obtained in \cite{duan2018DNN} as well.
    
    An interesting observation from numerical experiment is that using traditional $k$-NN leads to an even better performance than the $k$-NN trained via noise injection method. This observation contradicts to common belief that injecting attack into training algorithm to obtain an adversarial robust algorithm (e.g. optimization method in \citealp{sinha2018certifying}). 
    Therefore, it deserves further theoretical investigation to understand that, under which circumstance one can indeed benefit from the noise injection strategy. 

	\bibliographystyle{asa}
	\bibliography{VaRHDIS}

\newpage
\appendix
\section{Comparison between Random Perturbation and Non-random Perturbation}\label{sec:attack}
We use the following adversarial attack as the non-random perturbation:
	\begin{equation}\label{eqn:attack}
		\widetilde{x}=\begin{cases}
		\argmin\limits_{z\in B(x,\omega)} \eta(z)\qquad \mbox{if }\eta(x)>1/2\\
		\argmax\limits_{z\in B(x,\omega)} \eta(z)\qquad \mbox{if }\eta(x)\leq 1/2
		\end{cases}.
		\end{equation} 
	
	When $\omega\rightarrow 0$, if $\eta$ is differentiable, the length o attack converges to $\omega$ as well.
	
	The proposed attack scheme (\ref{eqn:attack}) is also called as ``white-box attack'' as the adversary has the knowledge of $\eta(x)$. On the other hand, unlike the ``white-box attack'' mentioned in \cite{wang2017analyzing}, the perturbation and attack we focus on are independent with the training samples.

\begin{theorem}\label{thm:target}
	Under [A.1] to [A.3], if testing data is adversarially attacked and $1/\sqrt{k},\zeta \ll \omega$, then 
		\begin{equation}\label{eqn:target}
		\begin{split}
		&\mbox{Regret}(k,n,\omega)=\frac{B_1}{4k}+\frac{1}{2}\int_{\mathcal{S}}\frac{\|\dot{\Psi}(x_0)\|}{\|\dot{\eta}(x_0)\|^2} \left(b(x_0)^2\zeta(x_0)^2+2\omega^2\|\dot{\eta}(x_0)\|^2\right) d\text{Vol}^{d-1}(x_0)+Rem,
		\end{split}
		\end{equation}
		where  $Rem:=O(\omega/\sqrt{k}+\omega\zeta)+o((1/k)\vee(\zeta+\omega)^2)$.
	\end{theorem}

From Theorem \ref{thm:target}, one can see that the regret under adversarial attack is larger than the one under random perturbation if the $\omega^2$ term is dominant. 

\section{Proof of Regret Analysis in Section 2 and 3}
	\subsection{Theorem 1}
	This section contains the proof of Theorem 1 and Theorem S.\ref{thm:target}. The two proofs are similar, so for proof of Theorem S.\ref{thm:target}, we only present the part where the proof is different from Theorem 1.
	
	Define $R_{1}(x)$ to $R_{k}(x)$ as the unsorted distance from the  nearest $k$ neighbors to testing data point $x$, and $R_{k+1}(x)$ as the distance from the exact $(k+1)$-th nearest neighbor to $x$ itself. Similar as \cite{CD14}, conditional on the distance of the $(k+1)$-th neighbor, the first $k$ neighbors are i.i.d. random variables distributed within $B(x,R_{k+1}(x))$.
	
	In addition to $f_1$, $f_2$, and $\Psi$, we further denote $\bar{f}(x)$ as the density of $X$.

		\begin{proposition}[Lemma S.1 in \cite{sun2016stabilized}]\label{S.1}
			For any distribution function $G$ with density $g$,
			\begin{eqnarray*}
				\int_{\mathbb{R}} [G(-bu-a)-1_{\{ u<0 \}}]du &=& -\frac{1}{b}\left\{ a+\int_{\mathbb{R}}tg(t)dt\right\},\\
				\int_{\mathbb{R}} u[G(-bu-a)-1_{\{ u<0 \}}]du &=& \frac{1}{b^2}\left\{  \frac{a^2}{2}+\frac{1}{2}\int_{\mathbb{R}}t^2g(t)dt +a\int_{\mathbb{R}}tg(t)dt\right\}.
			\end{eqnarray*}
		\end{proposition}
		Now we start our proof of Theorem 1.
	\begin{proof}[Proof of Theorem 1]
	The idea of proof follows \cite{samworth2012optimal}, and there are total 5 steps in our proof:
	\begin{itemize}
	    \item \textit{Step 0:} We give some definitions prepare for the proof.
	    \item \textit{Step 1:} Given a fixed (unobserved) testing sample $x$ and conditional on the perturbation random variable $\delta$, we obtain the mean and variance of $\widehat{\eta}_{k,n}(x+\delta)$. In particular, for any $x_0$ satisfying $\eta(x_0)=1/2$, let $x^t_0=x_0+t\frac{\dot{\eta}(x_0)}{\|\dot{\eta}(x_0)\|}$, we have that 
	    \begin{eqnarray*}
	    \mathbb{E}[\widehat{\eta}_{k,n}(x_0^t+\delta)|\delta]=\eta(x_0) + t\|\dot{\eta}(x_0)\| + \delta^{\top}\dot{\eta}(x_0)+ b(x_0)R_1^2+ O(t^2+\omega^2+R_1^4),
	    \end{eqnarray*}
	    and
	    \begin{eqnarray*}
	    Var(\widehat{\eta}_{k,n}(x_0^t+\delta)|\delta)=\frac{1}{4k}+O(\epsilon^2/k):=\frac{1}{s_{k,n}^2}+O(\epsilon^2/k).
	    \end{eqnarray*}
	    \item \textit{Step 2:} use tube theory to construct a tube based on $\mathcal{S}$. The remainder of regret outside the tube is of $O(\epsilon^3)$ for some $\epsilon$:
	    \begin{eqnarray*}
	            2*\mbox{Regret}&=
				&\int_{\mathbb{R}^d} \left( P\left( \sum_{i=1}^k \frac{1}{k}Y_i\leq \frac{1}{2}\bigg|\delta \right)-1_{\{ \eta(x)<1/2 \}} \right)dP(x) \nonumber\\&=& \int_{\mathcal{S}} \int_{-\epsilon_{k,n,\omega}}^{\epsilon_{k,n,\omega}} t\|\dot{\Psi}(x_0)\| \mathbb{E}\left( P(\widehat{\eta}_{k,n}(x_0^t+\delta)<1/2) -1_{\{t<0 \}} \right)dtd\text{Vol}^{d-1}(x_0)+r_1.
			\end{eqnarray*}
		for some remainder term $r_1$.
	    \item \textit{Step 3:} use Berry-Esseen Theorem to transform the probability $$P(\widehat{\eta}_{k,n}(x_0^t+\delta)<1/2)$$ to a Gaussian probability.
	    \begin{equation}\label{eqn:gaussian}
	        \Phi\left( \frac{ \mathbb{E}(1/2-\widehat{\eta}_{k,n}(x_0+\delta))}{\sqrt{ Var(\widehat{\eta}_{k,n}(x_0^t+\delta))}}\bigg|\delta\right)+o
	    \end{equation}
	    \item \textit{Step 4:}  plug in the mean and variance of $\widehat{\eta}_{k,n}$ from \textit{Step 1} into (\ref{eqn:gaussian}), integrate in the formula in \textit{Step 2} on the tube (integrate over $t$).
	\end{itemize}
	
	\textit{Step 0:} In this step we introduce some definitions.

			Denote $\delta$ as the random perturbation, i,e., $\delta=\widetilde{x}-x$. Denote $X_1(x)$ to $X_k(x)$ be the $k$ unsorted neighbors of $x$ in the training samples and $Y_i(x)$ be the $Y$ value for the corresponding $X_i(x)$. When no confusion is caused, we drop the argument $x$ and use $X_i$ and $Y_i$ for abbreviation. Then the probability of classifying contaminated sample as 0 becomes
			\begin{eqnarray*}
				P\left( \widehat{\eta}_{k,n}(x+\delta)\leq \frac{1}{2}\bigg|\delta \right) &=& P\left( \sum_{i=1}^k(Y_i(x+\delta)-1/2)<0 \bigg|\delta\right).
			\end{eqnarray*}

			If we directly investigate in $P(\widehat{\eta}_{k,n}(x)\leq 1/2)$, the possible values of $\eta(x)$ can be $[0,1]$, but those $\eta(x)$'s which are far from $1/2$ will have little contribution on the regret. Hence we consider $x$ in $\mathcal{S}^{\epsilon}$ where
			\begin{eqnarray*}
                \mathcal{S}^{\epsilon}=\left\{ x_0^t:=x_0+ t\frac{\dot{\eta}(x_0)}{\|\dot{\eta}(x_0)\|}: x_0\in\mathcal{S}, |t|\leq \epsilon \right\}
			\end{eqnarray*}
			for $\epsilon=\epsilon_{k,n,\omega}$.
			
			Further define
			\begin{equation*}
			    \epsilon_{k,n,\omega}=\max_{x\in\mathcal{\mathcal{R}}} C\left(\frac{\log k}{\sqrt{k}}\vee (R_1^2(x)+\omega)\right)
			\end{equation*}
			for some large constant $C$.
			
		\textit{Step 1}: For the scenario of random perturbation, $\delta$ is a random variable uniformly distributed on sphere $B(x_0^t,\omega)$, we first evaluate $\mathbb{E}(\widehat{\eta}_{k,n}(x_0^t+\delta))$ and $Var(\widehat{\eta}_{k,n}(x_0^t+\delta))$ for given $x_0$ and $\delta$.
		\begin{equation*}
		\begin{split}
		 &\mathbb{E}[\widehat{\eta}_{k,n}(x_0^t+\delta)|\delta]=\mathbb{E}(Y_1(x_0^t+\delta)|\delta)=\mathbb{E}\eta(X_1|\delta)\\=& 
				\mathbb{E} \left(\eta(x_0^t+\delta)+ (X_1-x_0^t-\delta)^{\top} \dot{\eta}(x_0^t+\delta) + 1/2(X_1-x_0^t-\delta)^{\top}\ddot{\eta}(x_0^t+\delta) (X_1-x_0^t-\delta)\bigg|\delta \right)\\
				&+rem ,
			\end{split}	
			\end{equation*}
		where $rem$ is a remainder term due to the Taylor's expansion. Before discussing $rem$, we consider the dominant part of $\mathbb{E}\widehat{\eta}_{k,n}(x_0^t+\delta)$. Conditional on $\delta$ and $R_1(x_t^0+\delta)=\|X_1-x_0^t-\delta\|$, then the distribution of $X_1$ is on the sphere of $B(x_t^0+\delta, R_1)$.
			Denote the density of this distribution as $\bar{f}(x|x_0^t+\delta,R_1(x_0^t+\delta))$. 
			Also define $\bar{f}'(x|x_0^t+\delta,R_1(x_0^t+\delta))$ as the gradient of $\bar{f}(x|x_0^t+\delta,R_1(x_0^t+\delta))$.
			For simplicity, rewrite $R_1(x_0^t+\delta)$ as $R_1$. Then based on (A.1) and (A.3) for the smoothness of $\bar{f}$ and $\eta$, rewrite $ \bar{f}(x|x_0^t+\delta,R_1)$ as a Taylor expansion at $x_0^t+\delta$, and we have
			\begin{eqnarray*}
				&&\mathbb{E}((X_1-x_0^t-\delta)^{\top}\dot{\eta}(x_0^t+\delta) |\delta,R_1)\\ &=& \int_{\partial B} (x-x_0^t-\delta)^{\top}\dot{\eta}(x_0^t+\delta) \bar{f}(x|x_0^t+\delta,R_1)dx\\
				&=& \int_{\partial B} (x-x_0^t-\delta)^{\top}\dot{\eta}(x_0^t+\delta) \bigg[\bar{f}(x_0^t+\delta|x_0^t+\delta,R_{1})\\&&\qquad\qquad\qquad\qquad+\bar{f}'(x_0^t+\delta|x_0^t+\delta,R_{1})^{\top}(x-x_0^t-\delta)\\&&\qquad\qquad\qquad\qquad+\frac{1}{2}(x-x_0^t-\delta)^{\top}\bar{f}''(x_0^t+\delta|x_0^t+\delta,R_{1})(x-x_0^t-\delta) \\&&\qquad\qquad\qquad\qquad+ O(\|x-x_0^t-\delta\|_2^3)\bigg]dx
				\\
				&=& \int_{\partial B} (x-x_0^t-\delta)^{\top}\dot{\eta}(x_0^t+\delta) \bar{f}'(x_0^t+\delta|x_0^t+\delta,R_{1})^{\top}(x-x_0^t-\delta)dx+o\\
				&=& tr\left( \dot{\eta}(x_0^t+\delta) \bar{f}'(x_0^t+\delta|x_0^t+\delta,R_{1})^{\top} \int_{\partial B}(x-x_0^t-\delta) (x-x_0^t-\delta)^{\top}dx \right)+O(R_1^4),
			\end{eqnarray*}
			where $\int_{\partial B}$ denotes integration over sphere $\partial B(x_0^t+\delta,R_1)$  the first-order and third-order terms becomes 0.
			
			In addition,
			\begin{eqnarray*}
				&&tr\left(  \frac{1}{2}\ddot{\eta}(x_0^t+\delta)\mathbb{E} \left( (X_1-x_0^t-\delta)(X_1-x_0^t-\delta)^{\top}|R_{1}\right) \right)\\
				&=& tr\left(  \frac{1}{2}\ddot{\eta}(x_0^t+\delta)\int_{\partial B} (x-x_0^t-\delta)(x-x_0^t-\delta)^{\top} \bar{f}(x|x_0^t+\delta,R_{1})dx\right)\\
				&=& tr\left(  \frac{\bar{f}(x_0^t+\delta|x_0^t+\delta,R_{1})}{2}\ddot{\eta}(x_0^t+\delta)\int_{\partial B} (x-x_0^t-\delta)(x-x_0^t-\delta)^{\top} dx\right)\\&&+tr\left( \frac{1}{2} \ddot{\eta}(x_0^t+\delta)\int_{\partial B} (x-x_0^t-\delta)(x-x_0^t-\delta)^{\top}(x-x_0^t-\delta)^{\top}\bar{f}'(x_0^t+\delta|x_0^t+\delta,R_{1}) dx\right)\\&&+O(R_1^4).
			\end{eqnarray*}
    The term $rem$ in $\mathbb{E}(\widehat{\eta}_{k,n})$ can be tackled in a similar manner and $rem=O(R_1^4)$. Hence taking 
	    \begin{equation*}
	        b(x)=\frac{1}{\bar{f}(x)d}\left\{ \sum_{j=1}^d [\dot{\eta}_j(x)\dot{\bar{f}}_j(x)+\ddot{\eta}_{j,j}(x)\bar{f}(x)/2 ] \right\},
	    \end{equation*}
	    we have
			\begin{eqnarray*}
				\mathbb{E}(\widehat{\eta}_{k,n}|\delta,R_1)&=& \eta(x_0^t+\delta)+b(x_0^t+\delta)R_{1}^2+O(R_1^4)\\&=&\eta(x_0)+\frac{t}{\|\dot{\eta}(x_0)\|}\dot{\eta}(x_0)^{\top}\dot{\eta}(x_0) + \delta^{\top}\dot{\eta}(x_0) + O(t^2+\omega^2) \\&&+b(x_0)R_1^2 + R_1^2\frac{t}{\|\dot{\eta}(x_0)\|} \dot{\eta}(x_0)^{\top}\dot{b}(x_0) + R_1^2 \delta^{\top}\dot{b}(x_0)+O(R_1^4)\\
				&=& \eta(x_0) + t\|\dot{\eta}(x_0)\| + \delta^{\top}\dot{\eta}(x_0)+ b(x_0)R_1^2+ O(t^2+\omega^2+R_1^4).
			\end{eqnarray*}
			
		Denote $t_{k,n}(x_0^t+\delta)=\mathbb{E}R_1^2$, using arguments in Lemma 1 and Theorem 2 of \cite{xing2018statistical}, take $a_d= 2^d\Gamma(1 + 1/2)^d /\Gamma(1 + d/2)$, we obtain
		\begin{eqnarray*}
		  t_{k,n}(x_0^t+\delta)&=&\frac{1}{a_d^{2/d}\bar{f}(x_0^t+\delta)^{2/d}} \left(\frac{k}{n}\right)^{2/d} +o(t_{k,n}^2(x_0^t+\delta))\\&=& t_{k,n}(x_0)+ O\left(t\left(\frac{k}{n}\right)^{2/d} \right) + O\left(\omega\left(\frac{k}{n}\right)^{2/d} \right)+o(t_{k,n}^2(x_0^t+\delta)).
		\end{eqnarray*}
    	Further denote ${\mu}_{k,n,\omega}(x_0^t,\delta)=\eta(x_0)+t\|\dot{\eta}(x_0)\|+\delta^{\top}\dot{\eta}(x_0)+b(x_0)t_{k,n}(x_0)$, we obtain
			\begin{eqnarray*}
				\mathbb{E}(\widehat{\eta}_{k,n})={\mu}_{k,n,\omega}(x_0^t,\delta)+O(t^2+\omega^2+t_{k,n}^2)={\mu}_{k,n,\omega}(x_0^t,\delta)+O(\epsilon_{k,n,\omega}^2).
			\end{eqnarray*}

		In terms of $Var(\widehat{\eta}_{k,n}(x_0^t,\delta))$, fixing $R_{k+1}$, the $k$ neighbors are i.i.d. random variables in $B(x_0^t+\delta,R_{k+1})$,
		\begin{eqnarray*}
    Var(Y_1|R_{k+1},\delta)=\mathbb{E}(Y_1|R_{k+1},\delta)(1-\mathbb{E}(Y_1|R_{k+1}\delta))=\frac{1}{4}+O \left(\epsilon_{k,n,\omega}^2\right),
    \end{eqnarray*}
    when $R_{k+1}^2=O(t_{k,n}(x_0))$. Moreover, as \cite{CD14} and \cite{belkin2018overfitting} mentioned, the probability of $R_{k+1}\gg t_{k,n}(x_0)$ is an exponential tail, hence the overall variance becomes
    \begin{equation*}
        Var(Y_1|\delta) = \frac{1}{4}+O \left(\epsilon_{k,n,\omega}^2\right).
    \end{equation*}
This also implies that
\begin{equation*}
    |\sqrt{Var(Y_1|\delta)}-\sqrt{1/4}|= O(\epsilon_{k,n,\omega}).
\end{equation*}

	    \textit{Step 2:} 
		Since the density of $x$ is 0 when $x$ is not in support, $$\int_{\mathbb{R}^d} \left( P\left( \sum_{i=1}^k \frac{1}{k}Y_i\leq \frac{1}{2} \right)-1_{\{ \eta(x)<1/2 \}} \right)dP(x)$$ is equal to $$\int_{\mathcal{R}} \left( P\left( \sum_{i=1}^k \frac{1}{k}Y_i\leq \frac{1}{2} \right)-1_{\{ \eta(x)<1/2 \}} \right)dP(x),$$ where $\mathcal{R}$ is the support of $X$. Taking $\epsilon_{k,n,\omega}\geq-s_{k,n}\log s_{k,n}$, we have
			\begin{eqnarray}\label{eqn:tube}
				&&\int_{\mathbb{R}^d} \left( P\left( \sum_{i=1}^k \frac{1}{k}Y_i\leq \frac{1}{2}\bigg|\delta \right)-1_{\{ \eta(x)<1/2 \}} \right)dP(x) \nonumber\\&=& \int_{\mathcal{S}} \int_{-\epsilon_{k,n,\omega}}^{\epsilon_{k,n,\omega}} t\|\dot{\Psi}(x_0)\|  P(\widehat{\eta}_{k,n}(x_0^t+\delta)<1/2) -1_{\{t<0 \}} dtd\text{Vol}^{d-1}(x_0)+r_1.
			\end{eqnarray}
			
			 The result in  (\ref{eqn:tube}) adopts tube theory to transform the integration from $\mathbb{R}^d$ to $\mathbb{R}\times \mathcal{S}$. Denote the map $\phi\left(x_0,t\frac{\dot{\eta}(x_0)}{\|\dot{\eta}(x_0)\|}\right)=x_0^t$, then the pullback of the $d$-form $dx$ is given at $(x_0,t{\dot{\eta}(x_0)}/{\|\dot{\eta}(x_0)\|})$ by
			\begin{eqnarray*}
			det\left(\dot{\Psi}\left(x_0,t\frac{\dot{\eta}(x_0)}{\|\dot{\eta}(x_0)\|}\right)\right)dtd\text{Vol}^{d-1}(x_0).
			\end{eqnarray*}

			For $r_1$, it is composed of four parts: (1) the integral outside $\mathcal{S}^{\epsilon_{k,n,\omega}}$, (2) the difference between $\Psi(t)$ and $t\|\dot{\Psi}(x)\|$, (3) the difference between $\mathcal{S}^{\epsilon_{k,n,\omega}}$ and the tube generated using $\mathcal{S}$, and (4) the remainder of $det\left(\dot{\Psi}\left(x_0,t\frac{\dot{\eta}(x_0)}{\|\dot{\eta}(x_0)\|}\right)\right)$:
			\begin{equation}\begin{split}
				&r_1\\
				=&\int_{\mathbb{R}^d\backslash \mathcal{S}^{\epsilon_{k,n,\omega}}} \left( P\left( \sum_{i=1}^k \frac{1}{k}Y_i\leq \frac{1}{2}  \bigg|\delta\right)-1_{\{ \eta(x)<1/2 \}} \right)dP(x)\\
				+&\int_{\mathcal{S}^{\epsilon_{k,n,\omega}}} ( \Psi(x)- t\|\dot{\Psi}(x)\| )\left( P\left( \sum_{i=1}^k \frac{1}{k}Y_i\leq \frac{1}{2}  \bigg|\delta\right)-1_{\{ \eta(x)<1/2 \}} \right)dx\\
				+&\bigg[\int_{\mathcal{S}^{\epsilon_{k,n,\omega}}}  t\|\dot{\Psi}(x)\| \left( P\left( \sum_{i=1}^k \frac{1}{k}Y_i\leq \frac{1}{2}  \bigg|\delta\right)-1_{\{ \eta(x)<1/2 \}} \right)dx\\
				& \qquad- \int_{\mathcal{S}} \int_{-\epsilon_{k,n,\omega}}^{\epsilon_{k,n,\omega}} t\|\dot{\Psi}(x_0)\| \left( P(\widehat{\eta}_{k,n}(x_0^t+\delta)<1/2) -1_{\{t<0 \}} \right)det\left(\dot{\Psi}\left(x_0,t\frac{\dot{\eta}(x_0)}{\|\dot{\eta}(x_0)\|}\right)\right)dtd\text{Vol}^{d-1}(x_0)\bigg]\\
				+&\int_{\mathcal{S}} \int_{-\epsilon_{k,n,\omega}}^{\epsilon_{k,n,\omega}} t\|\dot{\Psi}(x_0)\| \left( P(\widehat{\eta}_{k,n}(x_0^t+\delta)<1/2) -1_{\{t<0 \}} \right)\left[det\left(\dot{\Psi}\left(x_0,t\frac{\dot{\eta}(x_0)}{\|\dot{\eta}(x_0)\|}\right)\right)-1\right]dtd\text{Vol}^{d-1}(x_0)\\
				:=& r_{11}+r_{12}+r_{13}+r_{14}.
				\end{split}
			\end{equation}
			In $r_1$, $r_{12}$ is of $O(\epsilon_{k,n,\omega}^3)$ since $\Psi(x)-t\dot{\Psi}(x)=O(t^2)$; $r_{13}$ is of $O(\epsilon_{k,n,\omega}^3)$ since the difference of volume is of $O(\epsilon_{k,n,\omega}^2)$. For $r_{11}$ in $r_1$:
				\begin{eqnarray*}
					0&\geq &\int_{\mathbb{R}^d\backslash\mathcal{S}^{\epsilon_{k,n,\omega}}\cap \{x|\eta(x)<1/2  \}} \left( P\left(  \sum_{i=1}^k \frac{1}{k}Y_i\leq \frac{1}{2}\bigg|\delta \right)-1_{\{ \eta(x)<1/2 \}}  \right)dP(x)\\
					&=& \int_{\mathbb{R}^d\backslash\mathcal{S}^{\epsilon_{k,n,\omega}}\cap \{x|\eta(x)<1/2  \}} \left( P\left(  \sum_{i=1}^k \frac{1}{k}\left(Y_i-\frac{1}{2}\right)-\mathbb{E}\left(Y_1-\frac{1}{2}\right)\leq -\mathbb{E}\left(Y_1-\frac{1}{2}\right)\bigg|\delta\right)-1_{\{ \eta(x)<1/2 \}}  \right)dP(x)\\
					&=& -\int_{\mathbb{R}^d\backslash\mathcal{S}^{\epsilon_{k,n,\omega}}\cap \{x|\eta(x)<1/2  \}}  P\left(  \sum_{i=1}^k \frac{1}{k}\left(Y_i-\frac{1}{2}\right)-\mathbb{E}\left(Y_1-\frac{1}{2}\right)> -\mathbb{E}\left(Y_1-\frac{1}{2}\right)\bigg|\delta\right)dP(x).
				\end{eqnarray*}
				From the definition of $\epsilon_{k,n,\omega}$, we know that for any $\delta$, $\inf\limits_{x\in\mathbb{R}^d\backslash \mathcal{S}^{\epsilon_{k,n,\omega}}}|\mathbb{E}Y(\widetilde{x}_{\omega})-1/2|\geq c_1\epsilon_{k,n,\omega}$ for some $c_1>0$. Using Berstein inequality, we have an upper bound as
				\begin{eqnarray*}
	        	&&\int_{\mathbb{R}^d\backslash\mathcal{S}^{\epsilon_{k,n,\omega}}\cap \{x|\eta(x)<1/2  \}}  P\left(  \sum_{i=1}^k \frac{1}{k}(Y_i-\frac{1}{2})-\mathbb{E}(Y_1-\frac{1}{2})> -\mathbb{E}(Y_1-\frac{1}{2})\bigg|\delta\right)dP(x)\\
				    &\leq& O(\exp(-c_2k\epsilon_{k,n,\omega}^2))=o(1/k^{3/2}),
    	\end{eqnarray*}
    	for $c_2>0$.
    	
			Similar result can be obtained for $\mathbb{R}^d\backslash\mathcal{S}^{\epsilon_{k,n,\omega}}\cap \{x|\eta(x)> 1/2\}$.
			
			For $r_{14}$ in $r_1$,
			\begin{eqnarray*}
			&&\int_{\mathcal{S}} \int_{-\epsilon_{k,n,\omega}}^{\epsilon_{k,n,\omega}} t\|\dot{\Psi}(x_0)\| \mathbb{E}_{\delta}\left( P(\widehat{\eta}_{k,n}(x_0^t+\delta)<1/2) -1_{\{t<0 \}} \right)\\&&\qquad\qquad\qquad\left[det\left(\dot{\Psi}\left(x_0,t\frac{\dot{\eta}(x_0)}{\|\dot{\eta}(x_0)\|}\right)\right)-1\right]dtd\text{Vol}^{d-1}(x_0)\\
			&=& \int_{\mathcal{S}} \int_{-\epsilon_{k,n,\omega}}^{\epsilon_{k,n,\omega}} t\|\dot{\Psi}(x_0)\| \mathbb{E}_{\delta}\left( P(\widehat{\eta}_{k,n}(x_0+\delta)<1/2) -1_{\{t<0 \}} \right)\\&&\qquad\qquad\qquad\left[det\left(\dot{\Psi}\left(x_0,t\frac{\dot{\eta}(x_0)}{\|\dot{\eta}(x_0)\|}\right)\right)-1\right]dtd\text{Vol}^{d-1}(x_0)\\
			&&+\int_{\mathcal{S}} \int_{-\epsilon_{k,n,\omega}}^{\epsilon_{k,n,\omega}} t\|\dot{\Psi}(x_0)\| \left[t \frac{\dot{\eta}(x_0)^{\top}}{\|\dot{\eta}(x_0)\|} \frac{\partial}{\partial x_0}\mathbb{E}_{\delta}\left( P(\widehat{\eta}_{k,n}(x_0+\delta)<1/2) -1_{\{t<0 \}} \right)\right]\\&&\qquad\qquad\qquad\left[det\left(\dot{\Psi}\left(x_0,t\frac{\dot{\eta}(x_0)}{\|\dot{\eta}(x_0)\|}\right)\right)-1\right]dtd\text{Vol}^{d-1}(x_0)\\
			\\&&+o\\
			&=& \int_{\mathcal{S}} \int_{-\epsilon_{k,n,\omega}}^{\epsilon_{k,n,\omega}} t\|\dot{\Psi}(x_0)\| \left[t \frac{\dot{\eta}(x_0)^{\top}}{\|\dot{\eta}(x_0)\|} \frac{\partial}{\partial x_0}\mathbb{E}_{\delta}\left( P(\widehat{\eta}_{k,n}(x_0+\delta)<1/2) -1_{\{t<0 \}} \right)\right]\\&&\qquad\qquad\qquad\left[det\left(\dot{\Psi}\left(x_0,t\frac{\dot{\eta}(x_0)}{\|\dot{\eta}(x_0)\|}\right)\right)-1\right]dtd\text{Vol}^{d-1}(x_0)+o\\
			&=& O(\epsilon_{k,n,\omega}^4).
			\end{eqnarray*}
			Finally $r_1=O(\epsilon_{k,n,\omega}^3)$.

			\textit{Step 3: }we continue the derivation of $P(\widehat{\eta}_{k,n}(X+\delta)\leq 1/2,X\in\mathcal{S}^{\epsilon_{k,n,\omega}}|\delta)$ (where $X$ denotes testing sample random variable) using $$\int_{\mathcal{S}} \int_{-\epsilon_{k,n,\omega}}^{\epsilon_{k,n,\omega}} t\|\dot{\Psi}(x_0)\| \mathbb{E}_{\delta}\left( P(\widehat{\eta}_{k,n}(x_0^t+\delta)<1/2) -1_{\{t<0 \}} \right)dtd\text{Vol}^{d-1}(x_0).$$ Since $\widehat{\eta}_{k,n}$ is obtained from $n$ i.i.d. samples (though for some samples their weight is 0), by non-uniform Berry-Esseen Theorem,
			\begin{eqnarray*}
			&&\int_{\mathcal{S}}\int_{-\epsilon_{k,n,\omega}}^{\epsilon_{k,n,\omega}} t\|\dot{\Psi}(x_0)\| \left( P(\widehat{\eta}_{k,n}(x_0^t+\delta)<1/2) -1_{\{t<0 \}}|\delta \right)dtd\text{Vol}^{d-1}(x_0)\\&=& \int_{\mathcal{S}} \int_{-\epsilon_{k,n,\omega}}^{\epsilon_{k,n,\omega}} t\|\dot{\Psi}(x_0)\|\left( \Phi\left( \frac{ k\mathbb{E}(1/2-Y_1)}{\sqrt{kVar(Y_1)}}\bigg|\delta\right)-1_{\{t<0 \}} \right)dtd\text{Vol}^{d-1}(x_0) + r_2\\
			&&+\int_{\mathcal{S}}\int_{-\epsilon_{k,n,\omega}}^{\epsilon_{k,n,\omega}} t\|\dot{\Psi}(x_0)\| \left( \mathbb{E}_{R_{k+1}}\Phi\left( \frac{ k\mathbb{E}(1/2-Y_1|R_{k+1})}{\sqrt{kVar(Y_1|R_{k+1})}}\bigg|\delta\right)-\Phi\left( \frac{ k\mathbb{E}(1/2-Y_1)}{\sqrt{kVar(Y_1)}}\bigg|\delta\right) \right)dtd\text{Vol}^{d-1}(x_0).
			\end{eqnarray*}
			where
			\begin{eqnarray*}
				\left|P(\widehat{\eta}_{k,n}(x_0^t+\delta)<1/2)- \Phi\left( \frac{ k\mathbb{E}(1/2-Y_1)}{\sqrt{kVar(Y_1)}}\bigg|\delta\right)\right|&\leq& \frac{c_3}{\sqrt{k}}\frac{1}{    1+ k^{3/2}\left|\mathbb{E}1/2-Y_1\right|^3
				},
			\end{eqnarray*} hence
				\begin{eqnarray*}
					r_2&\leq&  \int_{\mathcal{S}} \int_{-\epsilon_{k,n,\omega}}^{\epsilon_{k,n,\omega}}t\|\dot{\Psi}(x_0)\| \frac{c_3}{\sqrt{k}}\frac{1}{    1+ k^{3/2}\left|\mathbb{E}1/2-Y_1\right|^3
					}dtd\text{Vol}^{d-1}(x_0)\\
					&=&\frac{c_4}{\sqrt{k}}\int_{\mathcal{S}} \int_{|t|<s_{k,n}} t\|\dot{\Psi}(x_0)\| dtd\text{Vol}^{d-1}(x_0)\\
					&&+\frac{c_5}{\sqrt{k}}\int_{\mathcal{S}} \int_{s_{k,n}<|t|<\epsilon_{k,n,\omega}}t\|\dot{\Psi}(x_0)\| \frac{1}{    1+ k^{3/2} t^3
					}dtd\text{Vol}^{d-1}(x_0)\\
					&\leq&O\left( \frac{ s^2_{k,n}}{\sqrt{k}}\right) +\frac{c_5}{k}\int_{\mathcal{S}} \int_{s_{k,n}<|t|<\epsilon_{k,n,\omega}}\sqrt{k}t\|\dot{\Psi}(x_0)\| \frac{1}{ k^{3/2} t^3
					}dt\frac{d\sqrt{k}}{d\sqrt{k}}d\text{Vol}^{d-1}(x_0)\\
					&=&O\left( \frac{ s^2_{k,n}}{\sqrt{k}}\right) +\frac{c_5}{k^{3/2}}\int_{\mathcal{S}} \int_{1<\sqrt{k}t<\epsilon_{k,n,\omega}\sqrt{k}}\sqrt{k}t\|\dot{\Psi}(x_0)\| \frac{1}{ k^{3/2} t^3
					}dt\frac{d\sqrt{k}}{d\sqrt{k}}d\text{Vol}^{d-1}(x_0)\\
					&=& O\left( \frac{1}{k^{3/2}}\right).
				\end{eqnarray*}
		   Following the analysis in \textit{Step 4}, we can also obtain that
		   \begin{eqnarray*}
		   &&\int_{\mathcal{S}}\int_{-\epsilon_{k,n,\omega}}^{\epsilon_{k,n,\omega}} t\|\dot{\Psi}(x_0)\| \left( \mathbb{E}_{R_{k+1}}\Phi\left( \frac{ k\mathbb{E}(1/2-Y_1|R_{k+1})}{\sqrt{kVar(Y_1|R_{k+1})}}\bigg|\delta\right)-\Phi\left( \frac{ k\mathbb{E}(1/2-Y_1)}{\sqrt{kVar(Y_1)}}\bigg|\delta\right) \right)dtd\text{Vol}^{d-1}(x_0)\\&=&O(\epsilon_{k,n,\omega}^5\sqrt{k}) + O(\epsilon_{k,n,\omega}^3).
		   \end{eqnarray*}
		    
		\textit{Step 4:} Finally we integrate on Gaussian probabilities:
			\begin{eqnarray*}
				&& \int_{\mathcal{S}} \int_{-\epsilon_{k,n,\omega}}^{\epsilon_{k,n,\omega}} t\|\dot{\Psi}(x_0)\|\left( \Phi\left( \frac{ k\mathbb{E}(1/2-Y_1)}{\sqrt{kVar(Y_1)}}\bigg|\delta\right)-1_{\{t<0 \}} \right)dtd\text{Vol}^{d-1}(x_0) \\
				&=&\int_{\mathcal{S}}\int_{-\epsilon_{k,n,\omega}}^{\epsilon_{k,n,\omega}} t\|\dot{\Psi}(x_0)\|\left(\Phi\left(-\frac{t\|\dot{\eta}(x_0)\|}{\sqrt{s_{k,n}^2}} -\frac{ (b(x_0)t_{k,n}(x_0)+\delta^{\top}\dot{\eta}(x_0^t))}{\sqrt{s_{k,n}^2}}\right)-1_{\{t<0\}}\right)dtd\text{Vol}^{d-1}(x_0)\\&&+r_3\\
				&=&\int_{\mathcal{S}}\int_{\mathbb{R}} t\|\dot{\Psi}(x_0)\|\left(\Phi\left(-\frac{t\|\dot{\eta}(x_0)\|}{\sqrt{s_{k,n}^2}} -\frac{ (b(x_0)t_{k,n}(x_0)+\delta^{\top}\dot{\eta}(x_0))}{\sqrt{s_{k,n}^2}}\right)-1_{\{t<0\}}\right)dtd\text{Vol}^{d-1}(x_0)\\&&+r_3+r_4\\
				&=& B_1s_{k,n}^2+\int_{S}\frac{\|\dot{\Psi}(x_0)\|}{\|\dot{\eta}(x_0)\|^2} (b(x_0)t_{k,n}(x_0)+\delta^{\top}\dot{\eta}(x_0))^2d\text{Vol}^{d-1}(x_0) +r_3+r_4.
			\end{eqnarray*}
		The last step follows Proposition \ref{S.1}. For the small order terms,
			\begin{eqnarray*}
				r_3&=&\int_{\mathcal{S}} \int_{-\epsilon_{k,n,\omega}}^{\epsilon_{k,n,\omega}} t\|\dot{\Psi}(x_0)\|\bigg(\Phi\left( \frac{ k\mathbb{E}(1/2-Y_1)}{\sqrt{kVar(Y_1)}}\bigg|\delta\right)\\&&\qquad\qquad\qquad-\Phi\left(-\frac{t\|\dot{\eta}(x_0)\|}{\sqrt{s_{k,n}^2}} -\frac{ (b(x_0)t_{k,n}(x_0)+\delta^{\top}\dot{\eta}(x_0^t))}{\sqrt{s_{k,n}^2}}\right) \bigg)dtd\text{Vol}^{d-1}(x_0)\\
				&=& O(\epsilon_{k,n,\omega}^3).
			\end{eqnarray*}

	Through definition of $\epsilon_{k,n,\omega}$ we have
			\begin{eqnarray*}
				r_4&=&o(1/k^{3/2}).
			\end{eqnarray*}
    
    Finally we take expectation on $\delta$:
    \begin{eqnarray*}
    && \mathbb{E}_{\delta}\int_{S}\frac{\|\dot{\Psi}(x_0)\|}{\|\dot{\eta}(x_0)\|^2} (b(x_0)t_{k,n}(x_0)+\delta^{\top}\dot{\eta}(x_0))^2d\text{Vol}^{d-1}(x_0)\\
    &=&\int_{S}\frac{\|\dot{\Psi}(x_0)\|}{\|\dot{\eta}(x_0)\|^2} \mathbb{E}_{\delta}(b(x_0)t_{k,n}(x_0)+\delta^{\top}\dot{\eta}(x_0))^2d\text{Vol}^{d-1}(x_0)\\
    &=& \int_{S}\frac{\|\dot{\Psi}(x_0)\|}{\|\dot{\eta}(x_0)\|^2}\left(b(x_0)^2t_{k,n}^2(x_0) +2b(x_0)t_{k,n}(x_0)\mathbb{E}_{\delta}(\delta^{\top}\dot{\eta}(x_0))+\mathbb{E}_{\delta}(\delta^{\top}\dot{\eta}(x_0))^2 \right)d\text{Vol}^{d-1}(x_0)\\
    &=& \int_{S}\frac{\|\dot{\Psi}(x_0)\|}{\|\dot{\eta}(x_0)\|^2}\left(b^2(x_0)t_{k,n}^2(x_0) + \frac{\|\dot{\eta}(x_0)\|^2}{d}\omega^2 \right)d\text{Vol}^{d-1}(x_0).
    \end{eqnarray*}
    
	\end{proof}
	\subsection{Theorem 2}
	
\begin{proof}[Proof of Theorem 2]
When $|\eta(x)-1/2|<C\omega$ for some large constant $C>0$, $g$ and $\widetilde{g}$ will always be the same, thus
\begin{eqnarray}
&&P(\widetilde{g}(\widetilde{x})\neq Y)- P(g(x)\neq Y)\\
&=& \mathbb{E}_{\delta}\left[\int_{\mathcal{S}}\int_{-C\omega}^{C\omega}t\|\dot{\Psi}(x_0)\|\left(1_{\{\widetilde{\eta}(x_0^t+\delta)<1/2\} }- 1_{\{t<0 \}}\right)dtd\text{Vol}^{d-1}(x_0) \right]+O(\omega^4).
\end{eqnarray}
Moreover,
\begin{eqnarray}
\widetilde{\eta}(\widetilde{x})=\mathbb{E}(\eta(x)|\widetilde{x}\text{ is observed})=\eta(\widetilde{x})+b(\widetilde{x})\omega^2 + O(\omega^3).
\end{eqnarray}
As a result,
\begin{eqnarray}
\widetilde{\eta}(x_0^t+\delta)=\eta(x_0)+t\|\dot{\eta}(x_0)\|+\dot{\eta}(x_0)^{\top}\delta+b(x_0)\omega^2+O(t\omega^2)+O(\omega^3).
\end{eqnarray}
Plugging in $\widetilde{\eta}(x_0^t+\delta)$ into regret, we obtain that
\begin{eqnarray}
&&P(\widetilde{g}(\widetilde{x})\neq Y)- P(g(x)\neq Y)\\
&=& \mathbb{E}_{\delta}\left[\int_{\mathcal{S}}\int_{-C\omega}^{C\omega}t\|\dot{\Psi}(x_0)\|\left(1_{\{\widetilde{\eta}(x_0^t+\delta)<1/2\} }- 1_{\{t<0 \}}\right)dtd\text{Vol}^{d-1}(x_0) \right]+O(\omega^4)\\
&=& \mathbb{E}_{\delta}\left[\int_{\mathcal{S}}\int_{-C\omega}^{C\omega}t\|\dot{\Psi}(x_0)\|\left(1_{\{t<-\dot{\eta}(x_0)^{\top}\delta/\|\dot{\eta}(x_0)\|-b(x_0)\omega^2\} }- 1_{\{t<0 \}}\right)dtd\text{Vol}^{d-1}(x_0) \right]+O(\omega^4)\\
&=& \mathbb{E}\left[\int_{\mathcal{S}}\|\dot{\Psi}(x_0)\|\int_{-\dot{\eta}(x_0)^{\top}\delta/\|\dot{\eta}(x_0)\|-b(x_0)\omega^2 }^{0} tdtd\text{Vol}^{d-1}(x_0)  \right]+O(\omega^4)\\
&=&\int_{\mathcal{S}}\|\dot{\Psi}(x_0)\| \frac{\omega^2}{2d}d\text{Vol}^{d-1}(x_0)+O(\omega^4).
\end{eqnarray}
From this derivation, the dominant terms in the denominator and numerator of the quantity
\begin{eqnarray}\label{eqn:compare}
    \frac{P(Y\neq\widehat{g}_{n}(\widetilde{x}))-P(Y\neq \widetilde{g}(\widetilde{x}))}{ P(Y\neq\widehat{g}_{n}'(\widetilde{x}))-P(Y\neq \widetilde{g}(\widetilde{x})) }
    \end{eqnarray}
are both $\Theta(n^{-4/(d+4)})$ when $k$'s are chosen to be optimal respectively. Note that the multiplicative constants for numerator and denominator are both determined by $\delta$ and density of $X$, and converges to each other when $\omega\rightarrow0$. As $\omega\rightarrow 0$ when $n\rightarrow \infty$,  the difference on the densities vanishes, thus (\ref{eqn:compare}) converges to 1.
\end{proof}
	\subsection{Theorem S.1}
	\begin{proof}[Proof of Theorem S.1] The proof is similar with Theorem 1. Since the format of $r_1$ to $r_4$ are unchanged, one can show that they are small order terms in Theorem 3 as well. What is changed in the proof of Theorem 3 is ${\mu}_{k,n,\omega}(x)$:
		
		When $t<0$, we have
			\begin{eqnarray*}
				{\mu}_{k,n,\omega}(x_0^t)= \eta(x_0)+t\|\dot{\eta}(x_0)\|+\omega\|\dot{\eta}(x_0)\|+b(x_0)t_{k,n}(x)+o,
			\end{eqnarray*}
			while for $t>0$, 	\begin{eqnarray*}
				{\mu}_{k,n,\omega}(x_0^t) = \eta(x_0)+t\|\dot{\eta}(x_0)\|-\omega\|\dot{\eta}(x_0)\|+b(x_0)t_{k,n}(x)+o.
			\end{eqnarray*}
			
		Therefore,
		\begin{eqnarray*}
		&&\int_{\mathcal{S}} \int_{-\epsilon_{k,n,\omega}}^{\epsilon_{k,n,\omega}} t\|\dot{\Psi}(x_0)\|\left( \Phi\left( \frac{ k\mathbb{E}(1/2-Y_1)}{\sqrt{kVar(Y_1)}}\right)-1_{\{t<0 \}} \right)dtd\text{Vol}^{d-1}(x_0)\\
&=&	\int_{\mathcal{S}}\int_{\mathbb{R}} t\|\dot{\Psi}(x_0)\|\left(\Phi\left(-\frac{t\|\dot{\eta}(x_0)\|-sign(t)\omega\|\dot{\eta}(x_0)\|}{\sqrt{s_{k,n}^2}} -\frac{ b(x_0)t_{k,n}(x_0))}{\sqrt{s_{k,n}^2}}\right)-1_{\{t<0\}}\right)dtd\text{Vol}^{d-1}(x_0)+o\\
&=& \int_{\mathcal{S}}\int_{\mathbb{R}} t\|\dot{\Psi}(x_0)\|\left(\Phi\left(-\frac{t\|\dot{\eta}(x_0)\|+\omega\|\dot{\eta}(x_0)\|}{\sqrt{s_{k,n}^2}} -\frac{ b(x_0)t_{k,n}(x_0))}{\sqrt{s_{k,n}^2}}\right)-1_{\{t<0\}}\right)dtd\text{Vol}^{d-1}(x_0)+r_5+o\\
&=&\frac{B_1}{4k}+\frac{1}{2}\int_{S}\frac{\|\dot{\Psi}(x_0)\|}{\|\dot{\eta}(x_0)\|^2} \left(b(x_0)\mathbb{E}R_1(x)^2+\omega\|\dot{\eta}(x_0)\|\right)^2 d\text{Vol}^{d-1}(x_0)+ r_5+o.
		\end{eqnarray*}
		The remainder $r_5$ is not a small order term, but we can show that it is positive, and calculate its rate.
		\begin{eqnarray*}
		r_5=O\left(\frac{B_1}{4k}+\int_{S}\frac{\|\dot{\Psi}(x_0)\|}{\|\dot{\eta}(x_0)\|^2} \left(b(x_0)\mathbb{E}R_1(x)^2+\omega\|\dot{\eta}(x_0)\|\right)^2 d\text{Vol}^{d-1}(x_0)\right).
		\end{eqnarray*}
		For $r_5$,
		\begin{eqnarray*}
		r_5&=& \int_{\mathcal{S}}\int_0^{+\infty} t\|\dot{\Psi}(x_0)\|\Phi\left(-\frac{t\|\dot{\eta}(x_0)\|-\omega\|\dot{\eta}(x_0)\|}{\sqrt{s_{k,n}^2}} -\frac{ b(x_0)t_{k,n}(x_0)}{\sqrt{s_{k,n}^2}}\right)dtd\text{Vol}^{d-1}(x_0)\\&&-\int_{\mathcal{S}}\int_0^{+\infty} t\|\dot{\Psi}(x_0)\|\Phi\left(-\frac{t\|\dot{\eta}(x_0)\|+\omega\|\dot{\eta}(x_0)\|}{\sqrt{s_{k,n}^2}} -\frac{ b(x_0)t_{k,n}(x_0)}{\sqrt{s_{k,n}^2}}\right)dtd\text{Vol}^{d-1}(x_0)\\
		&=& \int_{\mathcal{S}}\int_{-2\omega}^{+\infty} (t+2\omega)\|\dot{\Psi}(x_0)\|\Phi\left(-\frac{t\|\dot{\eta}(x_0)\|+\omega\|\dot{\eta}(x_0)\|}{\sqrt{s_{k,n}^2}} -\frac{ b(x_0)t_{k,n}(x_0)}{\sqrt{s_{k,n}^2}}\right)dtd\text{Vol}^{d-1}(x_0)\\&&-\int_{\mathcal{S}}\int_0^{+\infty} t\|\dot{\Psi}(x_0)\|\Phi\left(-\frac{t\|\dot{\eta}(x_0)\|+\omega\|\dot{\eta}(x_0)\|}{\sqrt{s_{k,n}^2}} -\frac{ b(x_0)t_{k,n}(x_0)}{\sqrt{s_{k,n}^2}}\right)dtd\text{Vol}^{d-1}(x_0)\\
		 &=&  \int_{\mathcal{S}}\int_{0}^{\infty} 2\omega\|\dot{\Psi}(x_0)\|\Phi\left(-\frac{t\|\dot{\eta}(x_0)\|+\omega\|\dot{\eta}(x_0)\|}{\sqrt{s_{k,n}^2}} -\frac{ b(x_0)t_{k,n}(x_0)}{\sqrt{s_{k,n}^2}}\right)dtd\text{Vol}^{d-1}(x_0)\\
		 &&+ \int_{\mathcal{S}}\int_{-2\omega}^0 (t+2\omega)\|\dot{\Psi}(x_0)\|\Phi\left(-\frac{t\|\dot{\eta}(x_0)\|+\omega\|\dot{\eta}(x_0)\|}{\sqrt{s_{k,n}^2}} -\frac{ b(x_0)t_{k,n}(x_0)}{\sqrt{s_{k,n}^2}}\right)dtd\text{Vol}^{d-1}(x_0)\\
		 &:=& A+B.
		\end{eqnarray*}
		From the format of $A$ and $B$, we know that they are positive. When $t_{k,n}(x_0)$ and $1/\sqrt{k}$ both $\ll \omega$, $A$ is an exponential tail (so we just ignore it) and for $B$ we have:
		\begin{eqnarray*}
		B = \int_{\mathcal{S}} \|\dot{\Psi}(x_0)\|\omega^2 d\text{Vol}^{d-1}(x_0) + O(\omega t_{k,n}(x_0)+\omega/\sqrt{k}).
		\end{eqnarray*}
		
	\end{proof}

	\subsection{Theorem 4}
	\begin{proof}[Proof of Theorem 4]
	First, it is easy to know that $\omega=O((1/n)^{1/d})$ since the nearest neighbor has an average distance of $O((1/n)^{1/d})$. 
	
	Second, there is a difference between pre-processed 1NN and random perturbation: in pre-processed 1NN, the nearest neighbor distributes approximately uniformly around $x$, while the other neighbors should have a distance to $x$ larger than the nearest neighbor. However, this difference only affects the remainder term of regret, i.e., assuming whether or not the other neighbors are uniformly distributed in the ball $B(x,R_{k+1})$ does not affect our result.

	As a result, taking expectation on the direction of $\delta$,
	\begin{eqnarray*}
	 && \mathbb{E}\int_{S}\frac{\|\dot{\Psi}(x_0)\|}{\|\dot{\eta}(x_0)\|^2} (b(x_0)t_{k,n}(x_0)+\delta^{\top}(x_0)\dot{\eta}(x_0))^2d\text{Vol}^{d-1}(x_0)\\
	&=& \int_{S}\frac{\|\dot{\Psi}(x_0)\|}{\|\dot{\eta}(x_0)\|^2} \mathbb{E}(b(x_0)t_{k,n}(x_0)+\delta^{\top}(x_0)\dot{\eta}(x_0))^2d\text{Vol}^{d-1}(x_0)\\
	&=&\int_{S}\frac{\|\dot{\Psi}(x_0)\|}{\|\dot{\eta}(x_0)\|^2}\left(b^2(x_0)t_{k,n}^2(x_0)  \right)d\text{Vol}^{d-1}(x_0) +\Theta(\omega^2).
	\end{eqnarray*}
	When $n^{-1/d}\gg n^{-2/(4+d)}$, i.e. $d > 4$, the dominant part of regret becomes $n^{-2/d}$.
	\end{proof}

	\section{Regret Convergence under General Smoothness Condition and Margin Condition}\label{sec:general}
	\subsection{Model and Theorem}
	
	In this section, we will
	relax the conditions on the distribution of $X$ and smoothness of $\eta$, and as a consequence, we only obtain the rate of the regret (without explicit form for the multiplicative constant). Technically, we will adopt the framework of \cite{CD14}, and the following assumptions on the smoothness of $\eta$ and the density of $X$ are used instead of conditions [A.1]-[A.3].
	
	\begin{enumerate}
	\item[B.1] Let $\lambda$ be the Lebesgue measure on $\mathbb{R}^d$. There exists a positive pair $(c_0,r_0)$ such that for any $x\in\mathcal{X}$,
	\begin{equation*}
	\lambda(\mathcal{X}\cap B(x,r))\geq c_0\lambda(B(x,r)),
	\end{equation*}
	for any $0<r\leq r_0$.
	\item[B.2] The support of $X$ is compact.
	\item[B.3] Margin condition:  $P(0<|\eta(x)-1/2|<t)\leq Bt^{\beta}$.
	\item[B.4] Smoothness of $\eta$: there exist some $\alpha>0$ and $c_r>0$, such that $|\eta(x+r)-\eta(x)|\leq \|r\|^{\alpha}$ for any $x$ and $r\leq c_r$.
	\item[B.5] The density of $X$ is finite and bounded away from 0.
\end{enumerate}
\begin{remark}
	The assumption B.3 is weaker from \cite{CD14} ($P(|\eta(x)-1/2|<t)\leq Bt^{\beta}$), but in fact does not affect the convergence.
	
	In \cite{CD14}, the assumption of smoothness is made on $|\mathbb{E}(\eta(x')|x'\in B(x,r))-\eta(x)|$, which is a weaker assumption compared with our B.4. However, under either random perturbation or adversarial attack, given a direction $\delta$ to obtain $\widetilde{x}$, the assumption in \cite{CD14} cannot be simply applied.
\end{remark}

The following theorem provide a general upper bound of regret for both perturbed and attacked data:

		\begin{theorem}[Convergence of Regret]\label{thm:general}
		Under [B.1] to [B.5], if for some $\delta>0$, $k/n^{\delta}\rightarrow \infty$, taking
		\begin{eqnarray*}
		k&\asymp& O(n^{2\alpha/(2\alpha+d)}\wedge (n^{2\alpha/d}\omega^{-2\alpha\beta})^{1/(2\alpha/d+\beta+1)}),
		\end{eqnarray*}
		the regret becomes
		\begin{eqnarray*}
		\mbox{Regret}(n,\omega)=O\left( \omega^{\alpha(\beta+1)}\vee n^{-\alpha(\beta+1)/(2\alpha+d)}\right),
		\end{eqnarray*}
		where $n^{-\alpha(\beta+1)/(2\alpha+d)}$ is the minimax rate of regret in $k$-NN.
		 
	\end{theorem}
	
	
	Theorem \ref{thm:general} also reveals a sufficient condition when $k$-NN is consistent, i.e regret finally converges to 0: for both perturbed and attacked data, when $\omega=o(1)$, $k$-NN is still consistent using these two types of corrupted testing data.
	
	\begin{theorem}[Minimax Rate of Regret]\label{thm:lower}
	Let $\widehat{g}_n$ be an estimator of $g$, let $\mathcal{P}_{\alpha,\beta}$ be a set of distributions which satisfy [B.1] to [B.5], when $\alpha\leq 1$, there exists some $C>0$ such that
	\begin{equation}
	    \sup\limits_{P\in\mathcal{P}_{\alpha,\beta} } P(\widehat{g}_n(\widetilde{X})\neq Y)-P(g(X)\neq Y) \geq C (\omega^{\alpha(\beta+1)}\vee n^{-\frac{\alpha(\beta+1)}{2\alpha+d}}).
	\end{equation}
	The constant $C$ depends on $\alpha,\beta,d$ only.
	\end{theorem}

	Theorem \ref{thm:lower} reveals that, for any estimator of $g$, under either random perturbation or adversarial attack, the regret in the worst case is larger than $C(\omega^{\alpha(\beta+1)}\vee n^{-\frac{\alpha(\beta+1)}{2\alpha+d}})$. Theorem \ref{thm:general} and \ref{thm:lower} together shows that the kNN estimator reaches the optimal rate of regret.
	\subsection{Proofs}
	\begin{proof}[Proof of Theorem S.\ref{thm:general}]
		Let $p=k/n$. Denote $R_{k,n}(x)=P(\widehat{g}_{k,n}(x)\neq Y|x)$ and $R^*(x)=P(g(x)\neq Y)$, and $\mathbb{E}R_{k,n}(x)-R^*(x)$ as the excess risk. Define
		\begin{align*}
		\mathcal{X}^+_{p,\Delta,\omega}=\{x\in\mathcal{X}| \eta(x)>\frac{1}{2}, \forall x'\in B(x,\omega), {\eta}(x'+r)\geq &\frac{1}{2}+\Delta ,\forall \|r\|<r_{2p}(x)\},\\
		\mathcal{X}^-_{p,\Delta,\omega}=\{x\in\mathcal{X}|\eta(x)<\frac{1}{2},\forall x'\in B(x,\omega), {\eta}(x'+r)\leq &\frac{1}{2}-\Delta , \forall \|r\|<r_{2p}(x)\},
		\end{align*}
		with $r_{2p}$ as the distance from $x$ to its $2pn$th nearest neighbor, and the decision boundary area:
		\begin{equation*}
		\partial_{p,\Delta,\omega}=\mathcal{X}\setminus(\mathcal{X}^+_{p,\Delta,\omega}\cup\mathcal{X}^-_{p,\Delta,\omega}).
		\end{equation*}
		
		Given $	\partial_{p,\Delta,\omega}$, $\mathcal{X}^+_{p,\Delta,\omega}$, and $\mathcal{X}^-_{p,\Delta,\omega}$, similar with Lemma 8 in \cite{CD14}, the event of $g(x)\neq \widehat{g}_{k,n}(x)$ can be covered as:
		\begin{eqnarray*}
			1_{\{ g(x)\neq \widehat{g}_{k,n}(x) \}}&\leq& 1_{\{ x\in \partial_{p,\Delta,\omega}\}}\\
			&&+1_{\{  \max\limits_{i=1,...,k}R_i(\widetilde{x})\geq r_{2p}(x) \}}\\
			&&+1_{\{ |\widehat{\eta}_{k,n}(x)-\eta(x'+r)|\geq \Delta \}}.
		\end{eqnarray*}
		
		When ${\eta}(x'+r)>1/2$ for all $\|r\|\leq r_{2p}(x)$, and $x\in\mathcal{X}_{p,\Delta}^+$,  assume  $\widehat{\eta}_{k,n}(x)<1/2$, then  $${\eta}(x'+r)-\widehat{\eta}_{k,n}(x') >{\eta}(x'+r)-1/2\geq \Delta.$$
		The other two events are easy to figure out. 
		
		By \cite{CD14} and \cite{belkin2018overfitting}, $P(\max\limits_{i=1,...,k}R_i(x)\geq r_{2p}(x) )$ is of $O(\exp(-ck^2))$ for some $c>0$, hence it becomes a smaller order term if for some $\delta>0$, $k/n^{\delta}\rightarrow \infty$.
		
		In addition, from the definition of regret, assume $\eta(x)<1/2$, 
		\begin{eqnarray*}
			&&P(\widehat{g}(x)\neq Y|X=x)-\eta(x)\\&=& \eta(x)P(\widehat{g}(x)=0|X=x) + (1-\eta(x))P(\widehat{g}(x)=1|X=x) -\eta(x)\\
			&=&\eta(x)P(\widehat{g}(x)=g(x)|X=x)+ (1-\eta(x))P(\widehat{g}(x)\neq g(x)|X=x)-\eta(x)\\
			&=& \eta(x)-\eta(x)P(\widehat{g}(x)\neq g(x)|X=x)+ (1-\eta(x))P(\widehat{g}(x)\neq g(x)|X=x)-\eta(x)\\
			&=& (1-2\eta(x))P(\widehat{g}(x)\neq g(x)|X=x),
		\end{eqnarray*}
		similarly, when $\eta(x)>1/2$, we have
		\begin{eqnarray*}
			P(\widehat{g}(x)\neq Y|X=x)-1+\eta(x)&=& (2\eta(x)-1)P(\widehat{g}(x)\neq g(x)|X=x).
		\end{eqnarray*}
		As a result, the regret can be represented as
		\begin{eqnarray*}
			Regret(k,n,\omega) &=& \mathbb{E}\left( |1-2\eta(X)|P(g(X)\neq \widehat{g}_{k,n}(X)) \right).
		\end{eqnarray*}
		For simplicity, denote $p=k/n$. We then follow the proof of Lemma 20 of \cite{CD14}. Without loss of generality assume $\eta(x)>1/2$. For perturbation $\delta\in\mathbb{R}^d$, define
		\begin{eqnarray*}
			\Delta_0&=& \sup\limits_{x,\delta,\|r\|<r_{2p}(x)}|\eta(x+\delta+r)-\eta(x)|=O(\omega^{\alpha})+O((k/n)^{\alpha/d}),\\
			\Delta(x)&=&|\eta(x)-1/2|,\nonumber
		\end{eqnarray*}
		then we have $$\eta(x+\delta+r)\geq \eta(x)-\Delta_0=\frac{1}{2}+(\Delta(x)-\Delta_0), $$
		hence $x\in\mathcal{X}_{p,\Delta(x)-\Delta_0,\omega}^{+}$.
		
		From the definition of $R_{k,n}$ and $R^*$, when $\Delta(x)>\Delta_0$, we also have
		\begin{eqnarray*}
			&&\mathbb{E}R_{k,n}(x)-R^*(x)\\&\leq& 2\Delta(x) \bigg[P(r_{(k+1)}>v_{2p})+P\bigg( \sum_{i=1}^k\frac{1}{k}Y(X_{i})-{\eta}(x'+\delta+r)>\Delta(x)-\Delta_0 \bigg)\bigg]\\
			&\leq& 2\Delta(x) P\bigg( \sum_{i=1}^k\frac{1}{k}Y(X_{i})-{\eta}(x'+\delta+r)>\Delta(x)-\Delta_0 \bigg)+o\\
			&=& 2\Delta(x) \mathbb{E}_{\delta}\left[P\bigg( \sum_{i=1}^k\frac{1}{k}Y(X_{i})-{\eta}(x'+\delta+r)>\Delta(x)-\Delta_0 \bigg|\delta\bigg)\right]+o
		\end{eqnarray*}
		Considering the problem that the upper bound can be much greater than 1 when $\Delta(x)$ is small, we define $\Delta_i=2^i\Delta_0$, taking $i_0 = \min\{ i\geq 1|\; (\Delta_i-\Delta_0)^{2}>1/k \}$, using Berstein inequality, it becomes
		
		\begin{eqnarray*}
			\mathbb{E}R_{k,n}(X)-R^*(X)
			&=&\mathbb{E}(R_{k,n}(X)-R^*(X))1_{\{\Delta(X)\leq\Delta_{i_0}\}}\\
			&&+\mathbb{E}(R_{k,n}(X)-R^*(X))1_{\{\Delta(X)>\Delta_{i_0}\}}\\
			&\leq& 2\Delta_{i_0} P(\Delta(X)\leq \Delta_{i_0})+\exp(-k/8)\\
			&&
			+c_2\mathbb{E}\left[\Delta(X)1_{\{\Delta_{i_0}<\Delta(X)\}}  \exp(-c_1k(\Delta(x)-\Delta_0)^2)\right]\\
			&\leq& 2\Delta_{i_0} P(\Delta(X)\leq \Delta_{i_0})+\exp(-k/8)\\
			&&
			+c_2\mathbb{E}\left[\Delta(X)1_{\{\Delta_{i_0}<\Delta(X)\}}  \exp(-c_1k(\Delta(x)-\Delta_0)^2)\right].
		\end{eqnarray*}	
		When $i_0 = \min\{ i\geq 1|\; (\Delta_i-\Delta_0)^{2}>1/k \}$, the exponential tail will diminish fast, leading to 
		\begin{eqnarray*}
		&&\mathbb{E}\left[\Delta(X)1_{\{\Delta_{i_0}<\Delta(X)\}}   \exp(-c_1k(\Delta(x)-\Delta_0)^2)\right]\\&=&\sum_{i=i_0}^{\infty}\mathbb{E}\left[\Delta(X)1_{\{\Delta_{i}<\Delta(X)<\Delta_{i+1}\}}   \exp(-c_1k(\Delta(x)-\Delta_0)^2)\right]\\
		&\leq& \sum_{i=i_0}^{\infty}\Delta_{i+1}^{\beta+1} \exp(-c_1k(\Delta_i-\Delta_0)^2)\\
		&=& \sum_{i=i_0}^{\infty}\Delta_{0}^{\beta+1}2^{(i+1)(\beta+1)} \exp(-c_1k\Delta_0^2(2^{i}-1)^2)\\
		&\leq& c_3\Delta_{0}^{\beta+1}.
		\end{eqnarray*}
		
		Recall that $\Delta_{i_0}>\Delta_0$ and $\Delta_{i_0}^2>1/k$, hence when $\Delta_{i_0}^2=O(1/k)$, we can obtain the minimum upper bound
		\begin{eqnarray*}
			\mathbb{E}R_{k,n}(X)-R^*(X)= O(\Delta_{0}^{\beta+1}) + O\left( \left(\frac{1}{k}\right)^{(\beta+1)/2} \right).
		\end{eqnarray*}
	\end{proof}
	\begin{proof}[Proof of Theorem S.\ref{thm:lower}]
	The proof is similar as \cite{audibert2007fast} using technical details in \cite{audibert2004classification} for Assouad's method.
	    There are two scenarios we will consider. Define $C_0$, $C_1$ and $C_2$ as some suitable constants, we will first show for any $\omega\geq 0$,
	    \begin{equation}
	        \sup\limits_{P\in\mathcal{P}_{\alpha,\beta} } P(\widehat{g}_n(\widetilde{X})\neq Y)-P(g(X)\neq Y) \geq C_1  n^{-\frac{\alpha(\beta+1)}{2\alpha+d}}.
	    \end{equation}
	    Further, when $\omega> C_0n^{-\frac{1}{2\alpha+d}}$, our target is to show that
	    \begin{equation}
	          \sup\limits_{P\in\mathcal{P}_{\alpha,\beta} } P(\widehat{g}_n(\widetilde{X})\neq Y)-P(g(X)\neq Y) \geq C_2 \omega^{\alpha(\beta+1)}.
	    \end{equation}
	\textit{Case 1}: when $\omega\leq C_0n^{-\frac{1}{2\alpha+d}}$, the basic idea is to construct a distribution of $x$ and two distributions of $y|x$ such that, the Bayes classifiers from these two distributions of $y|x$ reverse with each other, but through sampling $n$ points, we cannot distinguish which distribution these $n$ samples chosen are from. For example, given $n$ samples from a normal distribution, statistically we cannot determine whether data are sampled from a zero-mean distribution, or a distribution with mean $1/\sqrt{n}$, thus any estimator based on data (either using clean testing data or corrupted testing data) can make a false prediction.

Assume $X$ distributed within a compact set in $[0,1]^{d}$. For an integer $q\geq 1$, consider the regular grid as
\begin{equation}
    G_q:=\left\{ \left( \frac{2k_1+1}{2q},...,\frac{2k_d+1}{2q}  \right):k_i\in\{ 0,...,q-1 \},i=1,...,d  \right\}.
\end{equation}
For any point $x$, denote $n_q(x)$ as the closest grid point in $G_q$, and define $\mathcal{X}'_1, ..., \mathcal{X}'_{q^d}$ as a partition of $[0,1]^{d}$ such that $x$ and $x'$ are in the same $\mathcal{X}'_i$ if and only if $n_q(x)=n_q(x')$.  Among all the $\mathcal{X}'_i$'s, select $m$ of them as $\mathcal{X}_1,...,\mathcal{X}_m$, and $\mathcal{X}_0:=[0,1]^d\backslash \cup_{i=1}^m\mathcal{X}_i$. 

Take $z_i$ as the center of $\mathcal{X}_i$ for $i=1,...,m$. When $x\in B(z_i,1/4q)$, set the density of $x$ as $\epsilon/\lambda[B(z_i,1/4q)]$ for some $\epsilon>0$, and the density of $x$ in  $\mathcal{X}_i\backslash B(z_i,1/4q)$ is set to be 0. Assume $x$ uniformly distributes in $\mathcal{X}_0$.

Let $u:\mathbb{R}^+\rightarrow\mathbb{R}^+$ be a nonincreasing infinitely differentiable function starting from 0 and satisfying $\alpha$-smoothness condition. Moreover, $u$ is $1$ in $[1/2,\infty)$. Denote $\psi$ and $\phi$ as 
\begin{equation}
    \psi(x):=C_{\psi}u(\|x\|),
\end{equation}
and
\begin{equation}
    \phi(x):=q^{-\alpha}\psi(q(x-n_q(x))).
\end{equation}

Through the above construction, if we take $\eta(x)=(1+\phi(x))/2$ or $\eta(x)=(1-\phi(x))/2$, and let $m=O(q^{  d-\alpha\beta})$, then when $\alpha\beta\leq d$, $\beta$ margin condition is also satisfied.

The construction above will also be applied in \textit{Case 2} (with difference on the choice of $q,\epsilon,u$). 

Now we apply Assouad's method to find the lower bound of regret. Denote $P_{jk}$ as a distribution such that $\eta(x)=(1+\phi(x))/2$ when $k=0$, $x\in\mathcal{X}_j$, and $\eta(x)=(1-\phi(x))/2$ when $k=1$, $x\in\mathcal{X}_j$, then we have for any estimator $\widehat{g}(x,Z_n)$ with $Z_n=(X_n,Y_n)$ as data, 
	\begin{eqnarray}
	&&\sup\limits_{k=0,1}\mathbb{E}_{X,Z_n,P_{jk}}1_{\{ \widehat{g}(X,Z_n)\neq g(X)\}}1_{\{ X\in \mathcal{X}_j \}}\\
&\geq& 	\frac{1}{2}\mathbb{E}_{X,Z_n,P_{j0}}1_{\{ \widehat{g}(X,Z_n)\neq g(X)\}}1_{\{ X\in \mathcal{X}_j \}}+\frac{1}{2}\mathbb{E}_{Z_n,P_{j1}}1_{\{ \widehat{g}(X,Z_n)\neq g(X)\}}1_{\{ X\in \mathcal{X}_j \}}\\
&=& 	\frac{1}{2}\mathbb{E}_{X,Z_n,P_{j0}}1_{\{ \widehat{g}(X,Z_n)\neq 0\}}1_{\{ X\in \mathcal{X}_j \}}+\frac{1}{2}\mathbb{E}_{Z_n,P_{j1}}1_{\{ \widehat{g}(X,Z_n)\neq 1\}}1_{\{ X\in \mathcal{X}_j \}}\\
&=& \frac{1}{2}\mathbb{E}_{X}1_{\{X\in \mathcal{X}_j \}}\mathbb{E}\left[  \mathbb{E}_{Z_n,P_{j0}} 1_{\{ \widehat{g}(x,Z_n)\neq 0 \}} +\mathbb{E}_{Z_n,P_{j1}} 1_{\{ \widehat{g}(x,Z_n)\neq 1 \}}  \bigg|X=x\right]\\
&=& \frac{1}{2}\mathbb{E}_{X}1_{\{X\in \mathcal{X}_j \}}\mathbb{E}\left[  \int 1_{\{ \widehat{g}(x,Z_n)\neq 0 \}}dP_{j0}(Z_n) + \int 1_{\{ \widehat{g}(x,Z_n)\neq 1 \}} dP_{j1}(Z_n)  \bigg|X=x\right] \\
&\geq& \frac{1}{2}\mathbb{E}_{X}1_{\{X\in \mathcal{X}_j \}}\mathbb{E}\left[  \int 1_{\{ \widehat{g}(x,Z_n)\neq 0 \}} +  1_{\{ \widehat{g}(x,Z_n)\neq 1 \}} (dP_{j0}(Z_n) \wedge dP_{j1}(Z_n)) \bigg|X=x\right]\\
&=& \frac{1}{2}\mathbb{E}_{X}1_{\{X\in \mathcal{X}_j \}}\mathbb{E}\left[  \int  (dP_{j0}(Z_n) \wedge dP_{j1}(Z_n)) \bigg|X=x\right]\\
&=& \frac{1}{2}\mathbb{E}_{X}1_{\{X\in \mathcal{X}_j \}} \int  (dP_{j0}(Z_n) \wedge dP_{j1}(Z_n)).
	\end{eqnarray}

Denote
\begin{equation}
    b_j:=\left[ 1-\mathbb{E}^2( \sqrt{1-\phi^2(X)}|X\in\mathcal{X}_j) \right]^{1/2},
\end{equation}
and
\begin{equation}
    b'_j:=(\mathbb{E}\phi(X)|X\in\mathcal{X}_j),
\end{equation}
then $\int  (dP_{j0}(Z_n) \wedge dP_{j1}(Z_n))=\Theta(1)$ through our design of $\mathcal{X}_j$ when $b_j= O(1/\sqrt{n \epsilon})$ by Lemma 5.1 in \cite{audibert2004classification}.

As a result, when $b_j=b$, $b'_j=b'$ for all $j=1,...,m$, and $b= O(1/\sqrt{n\epsilon})$, there exists some $C_3>0$ such that
\begin{eqnarray}
&& \sup\limits_{P\in\mathcal{P}} P(\widehat{g}(X,Z_n)\neq Y)-P(g(X)\neq Y)\\
&=&  \sup\limits_{P\in\mathcal{P}}\mathbb{E}|2\eta(X)-1|P(\widehat{g}(X,Z_n)\neq g(X))\\
&=&  \sup\limits_{P\in\mathcal{P}}\sum_{j=1}^m\mathbb{E}|2\eta(X)-1|P(\widehat{g}(X,Z_n)\neq g(X)) 1_{\{X\in\mathcal{X}_j\}}\\
&\geq& C_3 mb'\epsilon.
\end{eqnarray}

The regret is lower bounded as $C_1n^{-\alpha(\beta+1)/(2\alpha+d)}$ when taking $q=O(n^{1/(2\alpha+d)})$. Note that $\widehat{g}(x,Z_n)$ can be any classifier, which also includes those ``random" estimators when $x$ is perturbed / attacked.
	
	\textit{Case 2}: when $\omega> C_0n^{-\frac{1}{2\alpha+d}}$, we construct a distribution of $(x,y)$ such that, after injecting noise in it, there is some sets of $\tilde{x}$ where $P(g(x)=1|\tilde{x})$ and $P(g(x)=0|\tilde{x})$ are comparable, thus no matter which label is obtained from the estimator, it has a constant-level of probability to make false decision at this $\tilde{x}$.
	
	The construction is similar as \textit{Case 1}, and we take $q=\lfloor 2/\omega\rfloor$. For function $u$, here we let it increase from 0 and becomes 1 in $[1/4,\infty)$. For each pair $(\mathcal{X}_{j0},\mathcal{X}_{j1})$, take $\eta(x)=(1+\phi(x))/2$ when $x\in\mathcal{X}_{j0}$ and $\eta(x)=(1-\phi(x))/2$ when $x\in\mathcal{X}_{j1}$. The support of $x$ is $\mathcal{X}_0\cup (\bigcup_{i=1}^m B(z_i,3\omega/4))$. Take $m=O(\omega^{\alpha\beta-d})$ and $\epsilon=O(\omega^d)$, then both $\alpha$-smoothness condition and $\beta$-margin condition are satisfied.
	
	After injecting random noise on $x$, consider $\xi_j$ as the boundary between $\mathcal{X}_{j0}$ and $\mathcal{X}_{j1}$, then when $\tilde{x}$ is from $\{z\;|\; dist(z,\xi_j)<\omega/4,\;z\in\mathcal{X}_{j0}\cup \mathcal{X}_{j1}  \}$, $P(g(x)=1|\tilde{x})$ and $P(g(x)=0|\tilde{x})$ are in $[C_4,1-C_4]$ for some constant $C_4>0$. Thus the probability of any estimator to make a false decision at this $\tilde{x}$ is larger than $C_4$. In addition, the probability measure of $\cup_{j=1}^m\{z\;|\; dist(z,\xi_j)<\omega/4,\;z\in\mathcal{X}_{j0}\cup \mathcal{X}_{j1}  \}$ is greater than $C_5\omega^{\alpha\beta}$ for some constant $C_5>0$.  Thus the regret is greater than $C_5\omega^{\alpha\beta}C_{\phi}\omega^{\alpha}C_4=C_6\omega^{\alpha(\beta+1)}$.


	\end{proof}
\end{document}